\documentclass[runningheads]{llncs}
\usepackage{graphicx}
\usepackage{comment}
\usepackage{amsmath,amssymb} 
\usepackage{color}

\usepackage{epsfig}
\usepackage{subfigure}
\usepackage{footnote}
\usepackage{algorithm}
\usepackage[noend]{algpseudocode} 

\usepackage{cuted}
\usepackage{capt-of}
\usepackage{tablefootnote}
\newcommand{\reals}{\mathbb{R}}
\usepackage[font=small]{caption} 
\usepackage{booktabs} 
\usepackage{multirow} 

\newcommand{\figref}[1]{Figure~\ref{#1}}

\newcommand{\cC}{\mathcal{C}}
\newcommand{\cR}{\mathcal{R}}
\newcommand{\cF}{\mathcal{F}}

\newcommand{\vv}{\boldsymbol{v}}
\newcommand{\zz}{\boldsymbol{z}}
\newcommand{\uu}{\boldsymbol{u}}
\newcommand{\phiv}{\boldsymbol{\phi}}
\newcommand{\psiv}{\boldsymbol{\psi}}
\newcommand{\ignorebig}[1]{}
\newcommand{\secref}[1]{Section~\ref{#1}}
\renewcommand{\eqref}[1]{Eq.~\ref{#1}}
\DeclareMathOperator{\E}{\mathbb{E}}
\newcommand{\normm}[1]{\left\|#1\right\|}

\newcommand{\hardmethod}{\text{SGC}}
\newcommand{\softslowmethod}{\text{WSGC-E}}
\newcommand{\softfastmethod}{\text{WSGC-S}}

\begin{document}
\pagestyle{headings}
\mainmatter
\def\ECCVSubNumber{5328}  

\title{Learning Canonical Representations for \\ Scene Graph to Image Generation} 




\titlerunning{Learning Canonical Representations for Scene Graph to Image Generation}

\author{
Roei Herzig\inst{1\star\dagger}~~~
Amir Bar\inst{1\star}~~~ \\
Huijuan Xu\inst{2}~~~
Gal Chechik\inst{3}~~~
Trevor Darrell\inst{2}~~~
Amir Globerson\inst{1}~~~
}

\authorrunning{Herzig et al.}

\institute{$^1$Tel Aviv University~~~ $^2$UC Berkeley~~~ $^3$Bar-Ilan University, NVIDIA Research}


\maketitle

\renewcommand*{\thefootnote}{$\star$}
\setcounter{footnote}{1}
\footnotetext{Equal Contribution.}
\renewcommand*{\thefootnote}{$\dagger$}
\setcounter{footnote}{1}
\footnotetext{Work done while at the University of Berkeley California.}
\renewcommand*{\thefootnote}{\arabic{footnote}}
\setcounter{footnote}{0}

\begin{figure}[ht]
    \centering
    \includegraphics[width=\linewidth]{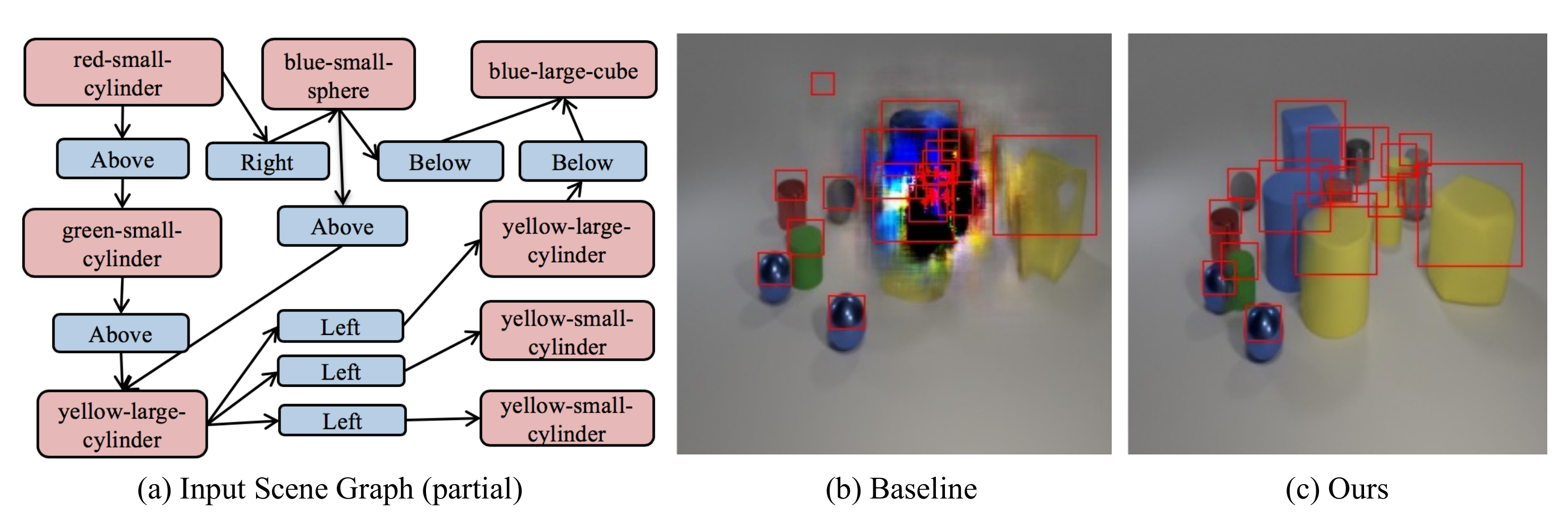}
    \captionof{figure}{{\bf Generation of scenes with many objects.} Our method achieves better performance on such scenes than previous methods. {\bf Left:} A partial input scene graph. {\bf Middle:} Generation using \cite{johnson2018image}. {\bf Right:} Generation using our proposed method.
    }
    \label{fig:teaser}
\end{figure}

\begin{abstract}
    Generating realistic images of complex visual scenes becomes challenging when one wishes to control the structure of the generated images. Previous approaches showed that scenes with few entities can be controlled using scene graphs, but this approach struggles as the complexity of the graph (the number of objects and edges) increases. In this work, we show that one limitation of current methods is their inability to capture semantic equivalence in graphs. We present a novel model that addresses these issues by learning canonical graph representations from the data, resulting in improved image generation for complex visual scenes.\footnote{The project page is available at https://roeiherz.github.io/CanonicalSg2Im/.} Our model demonstrates improved empirical performance on large scene graphs, robustness to noise in the input scene graph, and generalization on semantically equivalent graphs. Finally, we show improved performance of the model on three different benchmarks: Visual Genome, COCO, and CLEVR.

\end{abstract}


\section{Introduction}
\label{sec:intro}
Generating realistic images is a key task in computer vision research. Recently, a series of methods were presented for creating realistic-looking images of objects and faces (e.g. \cite{brock2018large,karras2019style,park2019semantic}). Despite this impressive progress, a key challenge remains: how can one control the content of images at multiple levels to generate images that have specific desired composition and attributes. Controlling content can be particularly challenging when generating visual scenes that contain multiple interacting objects. One natural way of describing such scenes is via the structure of a \textit{Scene Graph} (SG), which contains a set of objects as nodes and their attributes and relations as edges. Indeed, several studies addressed generating images from SGs \cite{ashual2019specifying,johnson2018image,li2019pastegan}. Unfortunately, the quality of images generated from SGs still lags far behind that of generating single objects or faces. Here we show that one problem with current models is their failure to capture logical equivalences, and we propose an approach for overcoming this limitation.


SG-to-image typically involves two steps: first, generating a layout from the SG, and then generating pixels from the layout. In the first step, the SG does not contain bounding boxes, and is used to generate a layout that contains bounding box coordinates for all objects. The transformation relies on geometric properties specified in the SG such as ``$(A,\mbox{right},B)$''. Since SGs are typically generated by humans, they usually do not contain {\em{all} }correct relations in the data. For example, in an SG with relation $(A,\mbox{right}, B)$ it is always true that $(B,\mbox{left},A)$, yet typically only one of these relations will appear.\footnote{We note that human raters don't typically include all logically equivalent relations. We analyzed data and found only small fraction of these are annotated in practice.} This example illustrates that multiple SGs can describe the same physical configuration, and are thus logically equivalent. Ideally, we would like all such SGs to result in the same layout and image. As we show here, this often does not hold for existing models, resulting in low-quality generated images for large graphs (see \figref{fig:teaser}).

Here we present an approach to overcome the above difficulty. We first formalize the problem as being invariant to certain logical equivalences (i.e., all equivalent SGs should generate the same image). Next, we propose to replace any SG with a ``canonical SG'' such that all logically equivalent SGs are replaced by the same canonical SG, and this canonical SG is the one used in the layout generation step. This approach, by definition, results in the same output for all logically equivalent graphs. We present a practical approach to learning such a canonicalization process that does not use any prior knowledge about the relations (e.g., it does not know that ``right'' is a transitive relation). We show how to integrate the resulting canonical SGs within a SG-to-image generation model, and how to learn it from data. Our method also learns more compact models than previous methods, because the canonicalization process distributes information across the graph with only few additional parameters.

In summary, our novel contributions are as follows: 1) We propose a model that uses canonical representations of SGs, thus obtaining stronger invariance properties. This in turn leads to generalization on semantically equivalent graphs and improved robustness to graph size and noise in comparison to existing methods. 2) We show how to learn the canonicalization process from data. 3) We use our canonical representations within an SG-to-image model and show that our approach results in improved generation on Visual Genome, COCO, and CLEVR, compared to the state-of-the-art baselines.

\section{Related Work}
\label{sec:related}
\textbf{Image generation}. Earlier work on image generation used autoregressive networks~\cite{van2016conditional,oord2016pixel} to model pixel conditional distributions. Recently, GANs \cite{goodfellow2014generative} and VAEs \cite{kingma2013auto} emerged as models of choice for this task. Specifically, generation techniques based on GANs were proposed for generating sharper, more diverse and better realistic images in a series of works~\cite{che2016mode,karras2019style,lim2017geometric,mao2017least,miyato2018spectral,radford2015unsupervised,salimans2016improved,wang2018pix2pixHD,zhang2018self,zhao2016energy}. 
\newline
\textbf{Conditional image synthesis}. Multiple works have explored approaches for generating images with a given desired content. Conditioning inputs may include class labels~\cite{chen2016infogan,mirza2014conditional,odena2017conditional}, source images~\cite{huang2018multimodal,pix2pix2017,liu2017unsupervised,taigman2016unsupervised,zhu2017unpaired,zhu2017toward}, model interventions~\cite{bau2019gandissect}, and text~\cite{hong2018inferring,qiao2019mirrorgan,reed2016generative,reed2016learning,sharma2018chatpainter,xu2018attngan,yin2019semantics,zhang2017stackgan}. Other studies~\cite{dong2017adversarial,nam2018mani} focused on image manipulation using language descriptions while disentangling the semantics of both input images and text descriptions.
\newline
\textbf{Structured representation}. Recent models~\cite{hong2018inferring,zhou2019text} incorporate intermediate structured representations, such as layouts or skeletons, to control the coarse structure of generated images. 
Several studies focused on  generating images from such representations (e.g., semantic segmentation masks~\cite{chen2017photographic,pix2pix2017,park2019semantic,wang2018pix2pixHD}, layout~\cite{zhao2019image}, and SGs~\cite{ashual2019specifying,johnson2018image,li2019pastegan}). Layout and SGs are more compact representations as compared to segmentation masks. While layout~\cite{zhao2019image} provides spatial information, SGs~\cite{johnson2018image} provide richer information about attributes and relations. Another advantage of SGs is that they are closely related to the semantics of the image as perceived by humans, and therefore editing an SG corresponds to clear changes in semantics. SGs and visual relations have also been used in image retrieval~\cite{johnson2015image,schuster2015generating}, relationship modeling~\cite{referential_relationships,raboh2020dsg,Schroeder2019ICCV}, image captioning~\cite{xu2019scene} and action recognition~\cite{herzig2019stag,CVPR2020_SomethingElse}. Several works have addressed the problem of generating SGs from text~\cite{schuster2015generating,tan2019text2scene}, standalone objects~\cite{xiaotian2019tell} and images \cite{herzig2018mapping}.
\newline
\textbf{Scene-graph-to-image generation}. Sg2Im~\cite{johnson2018image} was the first to propose an end-to-end method for generating images from scene graphs. However, as we note above, the current SG-to-image models ~\cite{ashual2019specifying,deng2018probabilistic,li2019pastegan,mittal2019interactive,Subarna2019ICLRW} show degraded performance on complex SGs with many objects. To mitigate this, the authors in \cite{ashual2019specifying} have utilized stronger supervision in the form of a coarse grid, where attributes of location and size are specified for each object. The focus of our work is to alleviate this difficulty by directly modeling some of the invariances in SG representation. Finally, the topic of invariance in deep architectures has also attracted considerable interest, but mostly in the context of certain permutation invariances \cite{herzig2018mapping,deep_sets}. Our approach focuses on a more complex notion of invariance, and addresses it via canonicalization. 

\begin{figure*}[t]
    \centering
    \includegraphics[width=\linewidth]{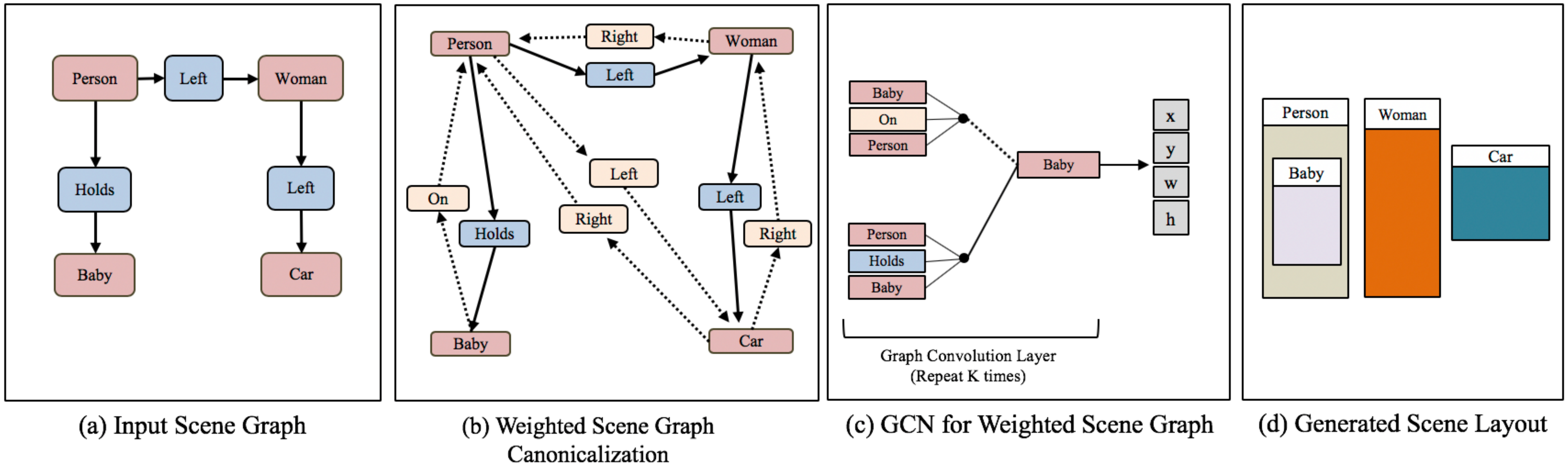}
    \captionof{figure}{{\bf Proposed Scene Graph to Layout architecture. }
    (a) An input scene graph. (b) The graph is first canonicalized using our WSGC method in \secref{sec:sg_to_weighted_sg}. Dashed edges correspond to completed relations that are assigned with weights. (c) A GCN is applied to the weighted graph, resulting in bounding box coordinates. (d) The GCN outputs are used to generate the predicted layout.}
    \label{fig:layout_arch}
\end{figure*}

\section{Scene Graph Canonicalization}
\label{sec:approach}
As mentioned above, the same image can be represented by multiple logically-equivalent SGs. Next we define this formally and propose an approach to canonicalize graphs that enforces invariance to these equivalences. In \secref{sec:sg_to_image} we show how to use this canonical scene graph within an SG-to-image task.

Let $\cC$ be the set of objects categories and $\cR$ be the set of possible relations.\footnote{Objects in SGs also contain attributes but we drop these for notational simplicity.}
An SG over $n$ objects is a tuple $(O,E)$ where $O\in\cC^n$ is the object categories and $E$ is a set of labeled directed edges (triplets) of the form $(i,r,j)$ where $i,j\in\{1,\ldots,n\}$ and $r\in\cR$. Thus an edge $(i,r,j)$ implies that the $i^{th}$ object (that has category $o_i$) should have relation $r$ with the $j^{th}$ object. Alternatively the set $E$ can be viewed as a set of $|\cR|$ directed graphs where for each $r$ the graph $E_r$ contains only the edges for relation $r$.

Our key observation is that relations in SGs are often dependent, because they reflect  properties of the physical world. This means that for a relation $r$, the presence of certain edges in $E_r$ implies that other edges have to hold. For example, assume $r$ is a \textbf{transitive relation} like ``left''. Then if $i,j\in E_r$ and $j,k \in E_r$, it should hold that $i,k\in E_r$. There are also dependencies between different relations. For example, if $r,r'$ are {\bf converse relations} (e.g., $r$ is ``left'' and $r'$ ``right'') then $i,j\in E_r$ implies $j,i\in E_{r'}$. Formally, all the above dependencies are first order logic formulas. For example, $r,r'$ being converse corresponds to the formula $\forall i,j: r(i,j)\implies r'(j,i)$. Let $\cF$ denote this set of formulas.

The fact that certain relations are implied by a graph does not mean that they are contained in its set of relations. For example, $E$ may contain $(1,{\mbox{left}}, 2)$ but not $(2,{\mbox{right}}, 1)$.\footnote{This is because empirical graphs $E$ are created by human annotators, who typically skip redundant edges that can be inferred from other edges.} However, we would like SGs that contain either or both of these relations to result in the same image. In other words, we would like all logically equivalent graphs to result in the same image, as formally stated next.

Given a scene graph $E$ denote by $Q(E)$ the set of graphs that are logically equivalent to $E$.\footnote{Equivalence of course depends on what relations are considered, but we do not specify this directly to avoid notational clutter.}  As mentioned above, we would like all these graphs to result in the same image. Currently, SG-to-layout architectures do not have this invariance property because they operate on $E$ and thus sensitive to whether it has certain edges or not. A natural approach to solve this is to replace $E$ with a {\em canonical form} $C(E)$ such that $\forall E'\in Q(E)$ we have $C(E')=C(E)$. There are several ways of defining $C(E)$. Perhaps the most natural one is the ``relation-closure'' which is the graph containing all relations implied by those in $E$. 
\begin{definition}
Given a set of formulas $\cF$, and relations $E$, the closure $C(E)$ is the set of relations that are true in any SG that contains $E$ and satisfies $\cF$.
\end{definition}
We note that the above definition coincides with the standard definition for closure of relations. Our definition emphasizes the fact that $C(E)$ are relations that are necessarily true given those in $E$. Additionally we allow for multiple relations, whereas closure is typically defined with respect to a single property.
Next we describe how to calculate $C(E)$ when $\cF$ is known, and then explain how to learn $\cF$ from data.



\subsection{Calculating Scene Graph Canonicalization \label{sec:weighted_sgc}}
For a general set of formulas, calculating the closure is hard as it is an instance of inference in first order logic. However, here we restrict ourselves to the following formulas for which this calculation is efficient:\footnote{We note that we could have added an option for symmetric relations, but we do not include these, as they not exhibited in the datasets we consider.}
\begin{itemize}
    \item Transitive Relations: We assume a set of relations $\cR_{trans}\subset \cR$ where all $r\in\cR_{trans}$ satisfy the formula $\forall x,y,z: r(x,y) \wedge r(y,z)\implies r(x,z)$.
    \item Converse Relations: We assume a set of relations pairs $\cR_{conv}\subset \cR\times \cR$ where all $(r,r')\in\cR_{conv}$ satisfy the formula $\forall x,y: r(x,y) \implies r'(y,x)$.
\end{itemize}

Under the above set of formulas, the closure $C(E)$ can be computed via the following procedure, which we call {\bf Scene Graph Canonicalization (\hardmethod)}:

{\noindent\bf Initialization:} Set $C(E)=E$. 
\newline
{\bf Converse Completion:}  $\forall (r,r')\in \cR_{conv}$, if $(i,r,j)\in E$, add $(j,r',i)$ to $C(E)$. 
\newline
{\bf Transitive Completion:} For each $r\in \cR_{trans}$ calculate
the transitive closure of $C_r(E)$ (namely the $r$ relations in $C(E)$) and add it to $C(E)$. The transitive closure can be calculated using the Floyd-Warshall algorithm \cite{floyd1962algorithm}.

It can be shown (see Supplementary) that the SGC procedure indeed produces the closure of $C(E)$. 


\subsection{Calculating Weighted Scene Graph Canonicalization \label{sec:sg_to_weighted_sg}}
Thus far we assumed that the sets $R_{trans}$ and $R_{conv}$ were given. Generally, we don't expect this to be the case. We next explain how to construct a model that doesn't have access to these. In this formulation we will add edges with weights, to reflect our level of certainty in adding them. These weights will depend on parameters, which will be learned from data in an end-to-end manner (see \secref{sec:losses}). See \figref{fig:layout_arch} for a high level description of the architecture.

Since we don't know which relations are transitive or converses, we assign probabilities to reflect this uncertainty. In the transitive case, for each $r\in \cR$ we use a parameter $\theta_r^{trans}\in \reals^{|\cR|}$ to define the probability that $r$ is transitive:
\begin{equation}
    p^{trans}(r) = \sigma(\theta^{trans}_r)
\label{eq:p_trans}
\end{equation}
where $\sigma$ is the sigmoid function. For converse relations, we let $p^{conv}(r'|r)$ denote the probability that $r'$ is the converse of $r$. We add another \textit{empty} relation $r'=\phi$ such that $p^{conv}(\phi|r)$ is the probability that $r$ has no converse in $\cR$. This is parameterized via $\theta_{r,r'}^{conv}\in \reals^{|\cR|\times |\cR\cup\phi|}$ which is used to define the distribution:
\begin{equation}
p^{conv}(r'|r) = \frac{e^{\theta^{conv}_{r, r'}}}{\sum_{\hat{r}\in{\cR\cup \phi}}{e^{\theta^{conv}_{r, \hat{r}}}}}
\label{eq:p_conv}
\end{equation}
Finally, since converse pairs are typically symmetric (e.g., ``left'' is the converse of ``right'' and vise-versa), for every $r, r' \in \cR \times \cR$ we set $\theta^{conv}_{r,r'} = \theta^{conv}_{r',r}$. Our model will use these probabilities to complete edges as explained next. In \secref{sec:weighted_sgc} we described the \hardmethod{} method, which takes a graph $E$ and outputs its completion $C(E)$. The method assumed knowledge of the converse and transitive relations. Here we extend this approach to the case where we have weights on the properties of relations, as per Equations~\ref{eq:p_trans} and \ref{eq:p_conv}. Since we have weights on possible completions we will need to work with a weighted relation graph and thus from now on consider edges $(i,r,j,w)$. Below we describe two methods \textit{WSGC-E} and \textit{WSGC-S} for obtaining weighted graphs. 
\secref{sec:sg_to_image} shows how to use these weighted graphs in an SG to image model.


\noindent {\bf Exact Weighted Scene Graph Canonicalization (\softslowmethod).}
We describe briefly a method that is a natural extension of \hardmethod~(further details are provided in the Supplementary).
It begins with the user-specified graph $E$, with weights of one. Next two weighted completion steps are performed, corresponding to the {\hardmethod} steps.
{\bf Converse Completion:} In \hardmethod, this step adds all converse edges. In the weighted case it makes sense to add the converse edge with its corresponding converse weight.
For example, if the graph $E$ contains the edge $(i,\text{above},j,1)$ and $p^{conv}(\text{below}|\text{above}) =0.7$, we add the edge $(j,\text{below},i,0.7)$. 
\noindent {\bf Transitive Completion:} In \hardmethod, all transitive edges are found and added. In the weighted case, a natural alternative is to set a weight of a path to be the product of weights along this path, and set the weight of a completed edge $(i,r,j)$ to be the maximum weight of a path between $i$ and $j$ times the probability $p^{trans}(r)$  that the relation is transitive. This can be done in poly-time, but runtime can be substantial for large graphs. We offer a faster approach next.


\noindent {\bf Sampling Based Weighted Scene Graph Canonicalization (\softfastmethod).} 
The difficulty in \softslowmethod{} is that the transitivity step is performed on a dense graph (most weights will be non-zero). To overcome this, we propose to replace the converse completion step of \softslowmethod{} with a sampling based approach that samples completed edges, but always gives them a weight of $1$ when they are added. In this way, the transitive step is computed on a much sparser graph with weights $1$. We next describe the two steps for the \softfastmethod{} procedure.

\noindent {\bf Converse Completion:} Given the original user-provided graph $E$, for each $r$ and edge $(i,r,j,1)$ we sample a random variable $Z\in\cR\cup\phi$ from $p^{conv}(\cdot|r)$ and if $Z\neq\phi$, we add the edge $(j,Z,i,1)$.  For example, see \figref{fig:WSGC3}b. After sampling such $Z$ for all edges, a new graph $E'$ is obtained, where all the weights are $1$.\footnote{We could sample multiple times and average, but this is not necessary in practice.}

\noindent {\bf Transitive Completion:} For the graph $E'$ and for each relation $r$, calculate the transitive closure of $C(E_r')$ and add all new edges in this closure to $E'$ with weight $p^{trans}(r)$. See illustration in \figref{fig:WSGC3}c. Note that this can be calculated in polynomial time using the FW algorithm \cite{floyd1962algorithm}, as in the \hardmethod{} case.

Finally, we note that if all assigned weights are discrete, both the \softslowmethod{} and \softfastmethod{} are identical to \hardmethod.

\begin{figure}[t!]  
\centering  
\ignorebig{
\begin{subfigure}
  \centering
    \includegraphics[width=\linewidth]{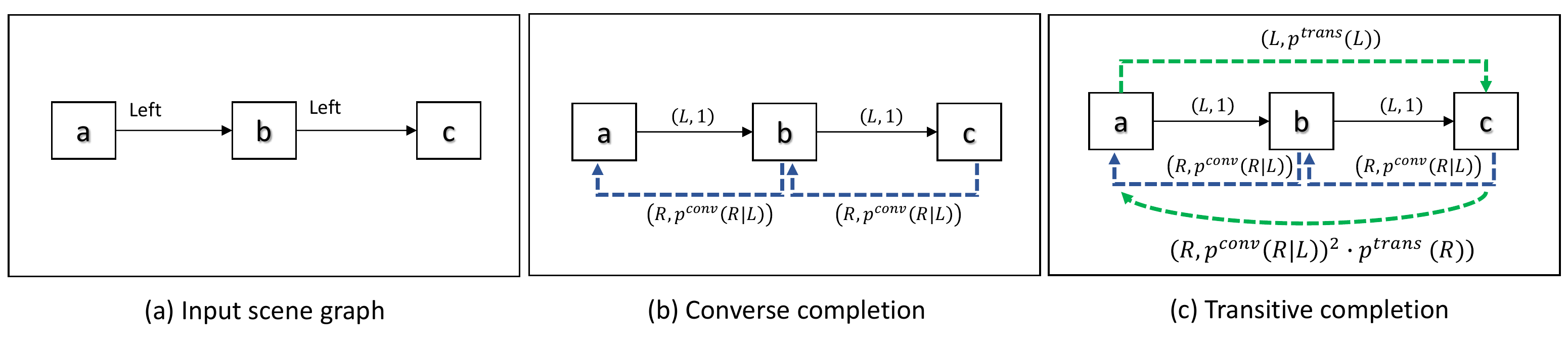}
    \caption{\small{An illustration of WSGC-E where relations are Left (L) and Right (R). (a) The input graph contains two relations with weight $1$. (b) Converse edges (blue dashed arrows) are completed with the weights $p^{conv}$. (c) Transitive edges (green dashed arrows) are added and assigned the weight of the corresponding path times $p^{trans}$.}}
    \label{fig:WSGC2}
\end{subfigure}  
}
\begin{subfigure}
  \centering
    \includegraphics[width=\linewidth]{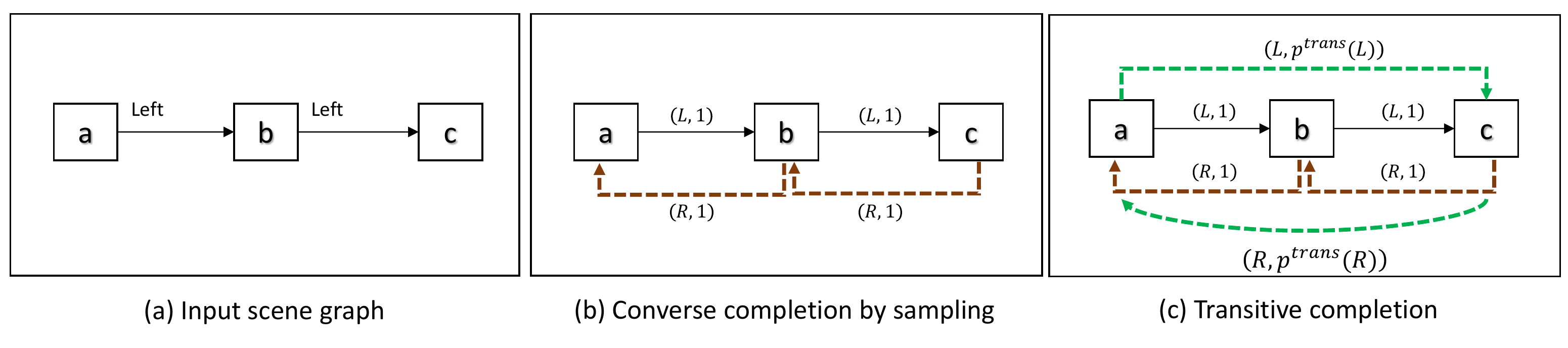}
    \caption{\small{Illustration of WSGC-S. (a) The input graph. (b) Converse edges (brown arrows) are sampled from $p^{conv}$ and assigned a weight $1$ (here two edges were sampled). (c) Transitive edges (green arrows) are completed and assigned a weight $p^{trans}$.}}
    \label{fig:WSGC3}
\end{subfigure}  

\end{figure}

\section{Scene Graph to Image using Canonicalization} 
\label{sec:sg_to_image}

Thus far we showed how to take the original graph $E$ and complete it into a weighted graph $E'$, using the {\softfastmethod} procedure. Next, we show how to use $E'$ to generate an image, by first mapping $E'$ to a scene layout (see \figref{fig:layout_arch}), and then mapping the layout to an image (see AttSPADE Figure in the Supplementary). The following two components are variants of previous SG to image models \cite{ashual2019specifying,johnson2018image,Sun_2019_ICCV}, and thus we describe them briefly (see Supplementary for details).

{\bf From Weighted SG to Layout}:
A layout is a set of bounding boxes for the nodes in the SG. A natural architecture for such graph-labeling problems is a Graph Convolutional Network (GCN) ~\cite{kipf2016semi}. Indeed, GCNs have recently been used for the SG to layout task \cite{ashual2019specifying,johnson2018image,li2019pastegan}. We also employ this approach here, but modify it to our weighted scene graph. Namely, we modify the graph convolution layer such that the aggregation step of each node is set to be a weighted average where the weights are those in the canonical SG.


{\bf From Layout to Image:}
We now need to transform the obtained layout in \secref{sec:sg_to_image} to an actual image.
Several works have proposed models for this step \cite{sun2019image,zhaobo2019layout2im}, where the input was a set of bounding boxes and their object categories.
We follow this approach, but extend it so that attributes for each object (e.g., color, shape and material, as in the CLEVR dataset) can be specified. We achieve this via a novel generative model, AttSPADE, that supports attributes. More details are in Supplementary. \figref{fig:attr_generation} shows an example of the model trained on CLEVR and applied to several SGs. Finally, our experiments on non CLEVR datasets simply we use a pre-trained LostGAN~\cite{Sun_2019_ICCV} model.


\section{Losses and Training \label{sec:losses}} 
Thus far we described a model that starts with an SG and outputs an image, using the following three steps: SG to canonical weighted SG (\secref{sec:sg_to_weighted_sg}), weighted SG to layout (\secref{sec:sg_to_image}) and finally layout to image (\secref{sec:sg_to_image}). In this section we describe how the parameters of these steps are trained in an end-to-end manner.
We focus on training with the {\softfastmethod}, since this is what we use in most of our experiments. See Supplementary~for Training with {\softslowmethod}.

Below we describe the loss for a single input scene graph $E$ and its ground truth layout $Y$. The parameters of the model are as follows: $\theta^{g}$ are the parameters of the GCN in \secref{sec:sg_to_image}, $\theta^{trans}$ are the parameters of the transitive probability (\eqref{eq:p_trans}), and $\theta^{conv}$ are those of the converse probability (\eqref{eq:p_conv}). Let $\theta$ denote the set of all parameters. Recall that in the first step \secref{sec:sg_to_weighted_sg}, we sample a set of random variables $\bar{Z}$ and use these to obtain a weighted graph $\text{WSGC}_{\bar{Z}}(E;\theta^{trans})$. Denote the GCN applied to this graph by $G_{\theta^g}(\text{WSGC}_{\bar{Z}}(E;\theta^{trans}))$.

We use the $L_1$ loss between the predicted and ground truth bounding boxes $Y$. Namely, we wish to minimize the following objective:
\begin{equation}
L(\theta) = \E_{\bar{Z} \backsim q(\theta^{conv})} \normm{Y - G_{\theta^{g}}(\text{WSGC}_{\substack{\bar{Z}}}(E;\theta^{trans}))}_1
\label{eq:opt_loss}
\end{equation}
where $\bar{Z} = \{Z_e | e \in E\}$ is a set of independent random variables each sampled from $p^{conv}(r'|r(e);\theta^{conv})$ (see \eqref{eq:p_conv} and the description of {\softslowmethod}), and $q(\theta^{conv})$ denotes this sampling distribution.

The gradient of this loss with respect to all parameters except $\theta^{conv}$ can be easily calculated. Next, we focus on the gradient with respect to $\theta^{conv}$. Because the sampling distribution depends on $\theta^{conv}$ it is natural to use the REINFORCE algorithm \cite{Williams92REINFORCE} in this case, as explained next. Define: 

$R(\bar{Z}; \theta^g,\theta^{trans}) = \normm{Y - G_{\theta^{g}}(\text{WSGC}_{\substack{\bar{Z}}}(E;\theta^{trans}))}_1$.
Then \eqref{eq:opt_loss} is:
$L(\theta^{conv}) = \E_{\substack{\bar{Z} \backsim     q(\theta^{conv})}}{R(\bar{Z}; \theta^g,\theta^{trans})}$.

The key idea in REINFORCE is the observation that:
$$\nabla_{\theta^{conv}}L(\theta) = \E_{\substack{\bar{Z} \backsim q({\theta}^{conv})}}{\nabla_{\theta^{conv}}{R(\bar{Z}; \theta^g,\theta^{trans})}\log{p^{conv}_{\theta}(\bar{Z})}}$$
Thus, we can approximate $\nabla_{\theta^{conv}}L(\theta)$ by sampling $\bar{Z}$ and averaging the above.\footnote{We sample just one instantiation of $\bar{Z}$ per image, since this works well in practice. }

For the layout-to-image component, most of our experiments use a pre-trained LostGAN model. For CLEVR (\figref{fig:attr_generation}) we train our AttSPADE model which is a variant of SPADE~\cite{park2019semantic} and trained similarly (see Supplementary).


\begin{figure*}[t!]
    \centering
    \includegraphics[width=\linewidth]{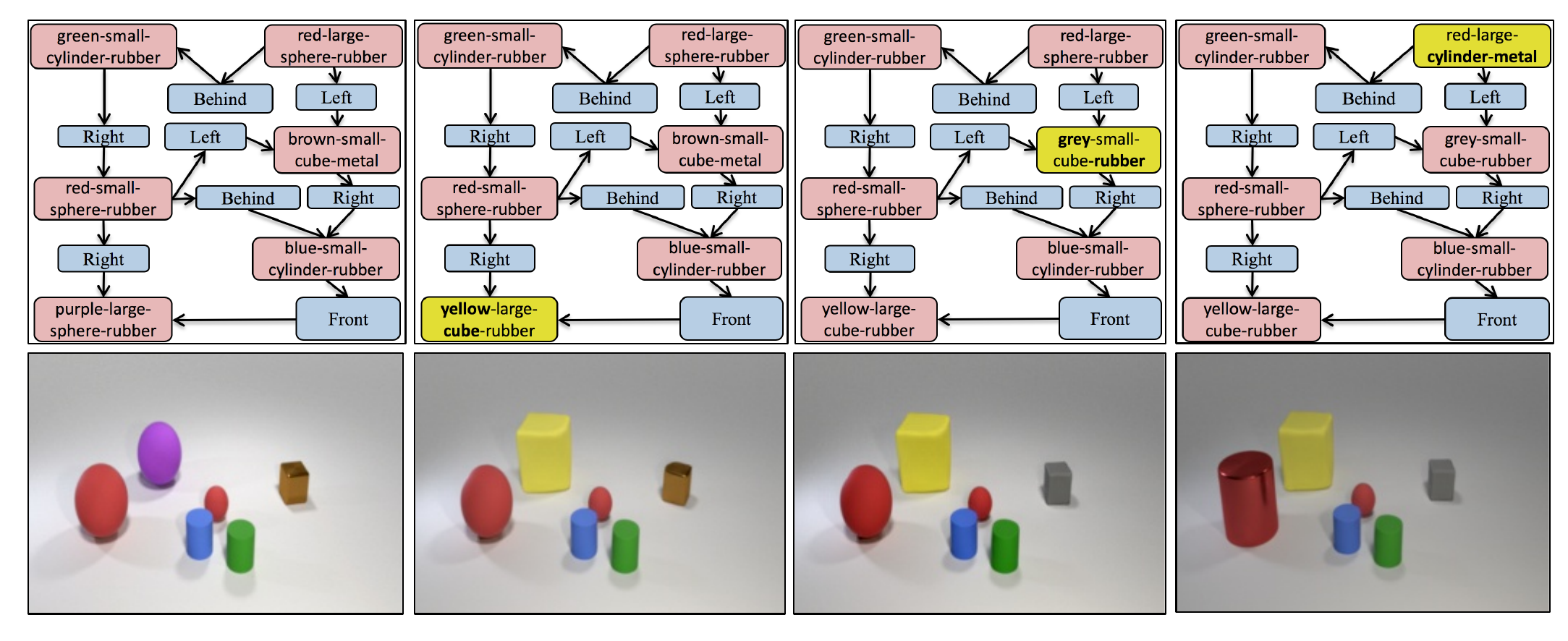}
    \caption{{Demonstration of the AttSPADE generator for scene graphs with varying attributes. Top row shows SGs where each column modifies one attribute. Bottom row is the images generated by AttSPADE.}}
    \label{fig:attr_generation}
\end{figure*}

\section{Experiments}
\label{sec:exp}
To evaluate our proposed WSGC method, we test performance on two tasks. First, we evaluate on the SG-to-layout task (the task that WSGC is designed for. See \secref{sec:sg_to_weighted_sg}). We then further use these layouts to generate images and demonstrate that improved layouts also yield improved generated images. 


{\bf Datasets.} 
We consider the following three datasets: COCO-stuff \cite{Caesar2018COCOStuffTA}, Visual Genome (VG) \cite{krishnavisualgenome} and CLEVR~\cite{johnson2017clevr}. We also created a synthetic dataset to quantify the performance of WSGC in a controlled setting. 
\newline
{\bf Synthetic dataset.}
\label{sec:toy_exp}
To test the contribution of learned transitivity to layout prediction, we generate a synthetic dataset. In this dataset, every object is a square with one of two possible sizes. The set of relations includes: $Above$ (transitive), $\mbox{Opposite Horizontally}$ and $X Near$ (non-transitive). To generate training and evaluation data, we uniformly sample coordinates of object centers and object sizes and automatically compute relations among object pairs based on their spatial locations. See Supplementary file for further visual examples. 

\noindent{\bf COCO-Stuff 2017~\cite{Caesar2018COCOStuffTA}.}
Contains pixel-level annotations with 40K train and 5K validation images with bounding boxes and segmentation masks for 80 thing categories, and 91 stuff categories. We use the standard subset proposed in previous works ~\cite{johnson2018image}, which contains $\sim$25K training, 1024 validation, and 2048 in test. We use an additional subset we call Packed COCO, containing images with at least 16 objects, resulting in $4,341$ train images, 238 validation, and 238 test.

\noindent{\bf Visual Genome (VG)~\cite{krishnavisualgenome}.} Contains $108,077$ images with SGs. We use the standard subset \cite{johnson2018image}: 62K training, 5506 validation and 5088 test images. We use an additional subset we call Packed VG, containing images with at least 16 objects, resulting in 6341 train images, 809 validation, and 809 test images.

\noindent{\bf CLEVR~\cite{johnson2017clevr}.} A synthetic dataset based on scene-graphs  with  four  spatial  relations:  $left$,  $right$,  $front$ and $behind$, as well as attributes $shape$, $size$, $material$ and $color$. It has 70k training images and 15k for validation and test. 


\begin{figure*}[t!]
    \centering
    \centering
    \includegraphics[width=\linewidth]{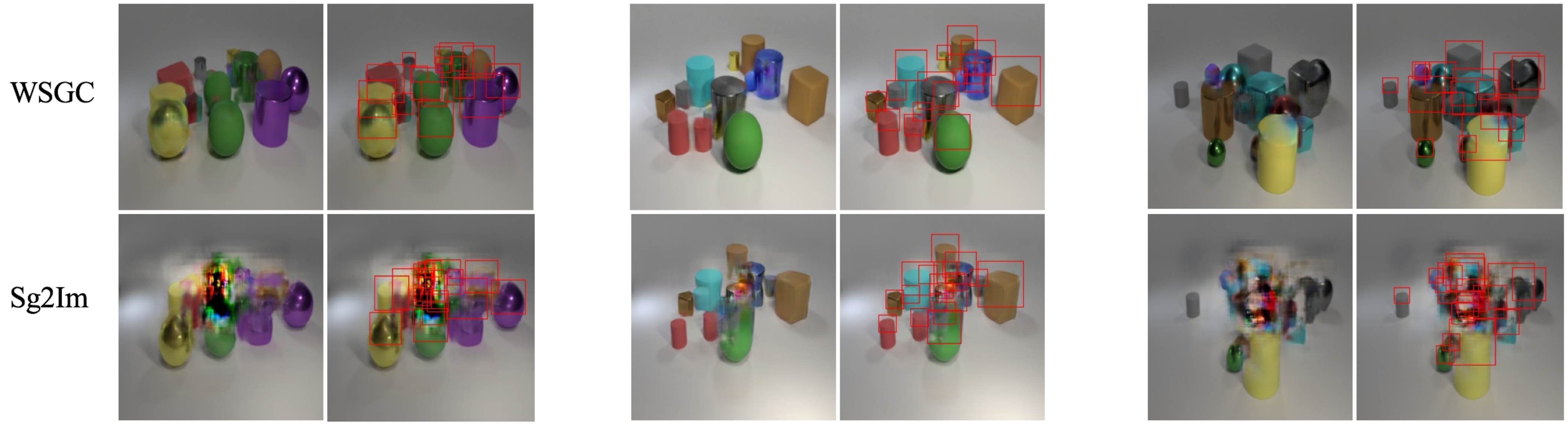}
    \caption{Examples of image generation for CLEVR where the Sg2Im baseline and our WSGC model were trained on images with a maximum of 10 objects but tested on scenes with 16+ objects. Shown are three examples where: Top row: our WSGC generation (with boxes and without). Bottom row:  Sg2Im generation (with boxes and without).}
    \label{fig:clevr_ashes}
\end{figure*}




\subsection{Scene-Graph-to-layout Generation}

We evaluate the SG-to-layout task using the following metrics: 1) $mIOU$: the mean IOU value. 2)  $R@0.3$ and $R@0.5$: the average recall over predictions with $IOU$ greater than $0.3$ and $0.5$ respectively. We note our WSGC model is identical to the Sg2Im baseline in the SG-to-layout module in all aspects that are not related to canonicalization. This provides a well-controlled ablation showing that canonicalization improves performance.
\newline
{\bf Testing Robustness to Number of Objects.} Scenes can contain a variable number of objects, and SG-to-layout models should work well across these. Here we tested how different models perform as the number of objects is changed in the synthetic dataset. We compared the following models a) A ``Learned Transitivity'' model that uses WSGC to learn the weights of each relation. b) A ``Known Transitivity'' model that is given the transitive relations in the data, and performs hard SGC completion (see \secref{sec:weighted_sgc}). Comparison between ``Learned Transitivity'' and ``Known Transitivity'' is meant to evaluate how well WSGC can learn which relations are transitive. c) A baseline model 
Sg2Im~\cite{johnson2018image} that does not use any relation completion, but otherwise has the same architecture.

\begin{figure*}[t!]
    \centering
    \begin{subfigure}[2 GCN Layers]{
    \includegraphics[width=0.3\linewidth]{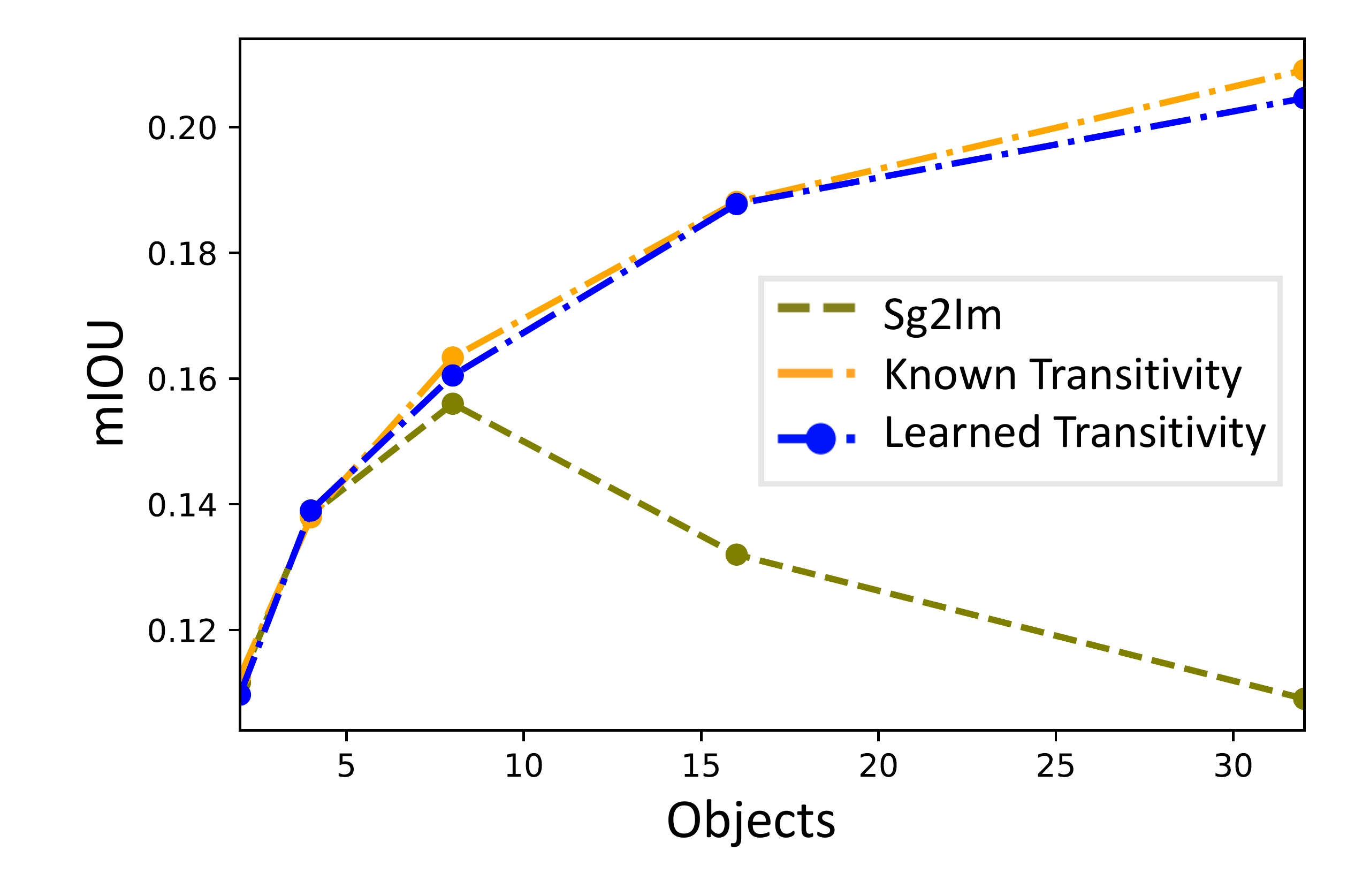}}
    \end{subfigure}\quad
    \begin{subfigure}[4 GCN Layers]{
    \includegraphics[width=0.3\linewidth]{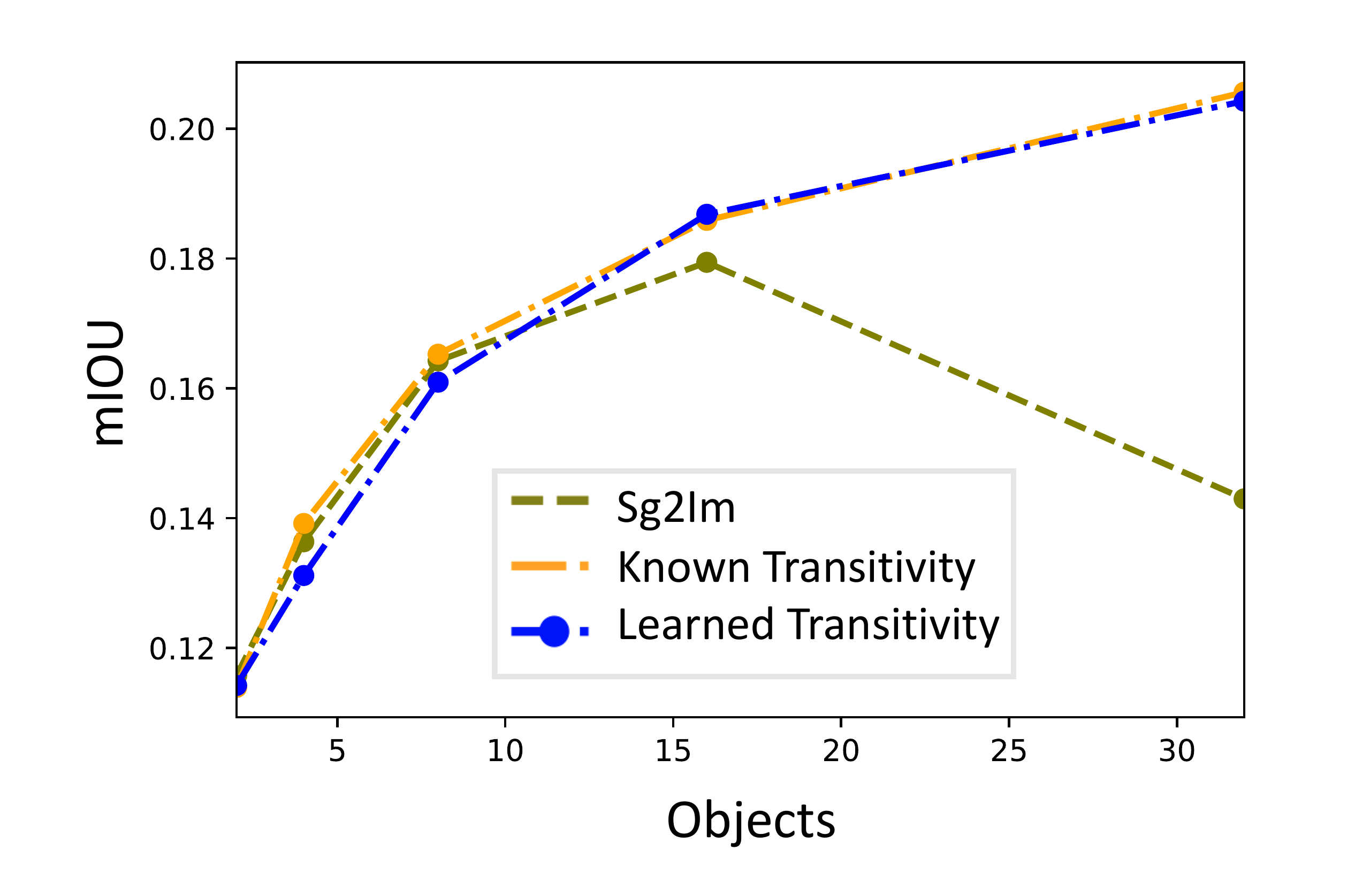}}
    \end{subfigure}\quad
    \begin{subfigure}[Trained on 16 objects]{
    \includegraphics[width=0.3\linewidth]{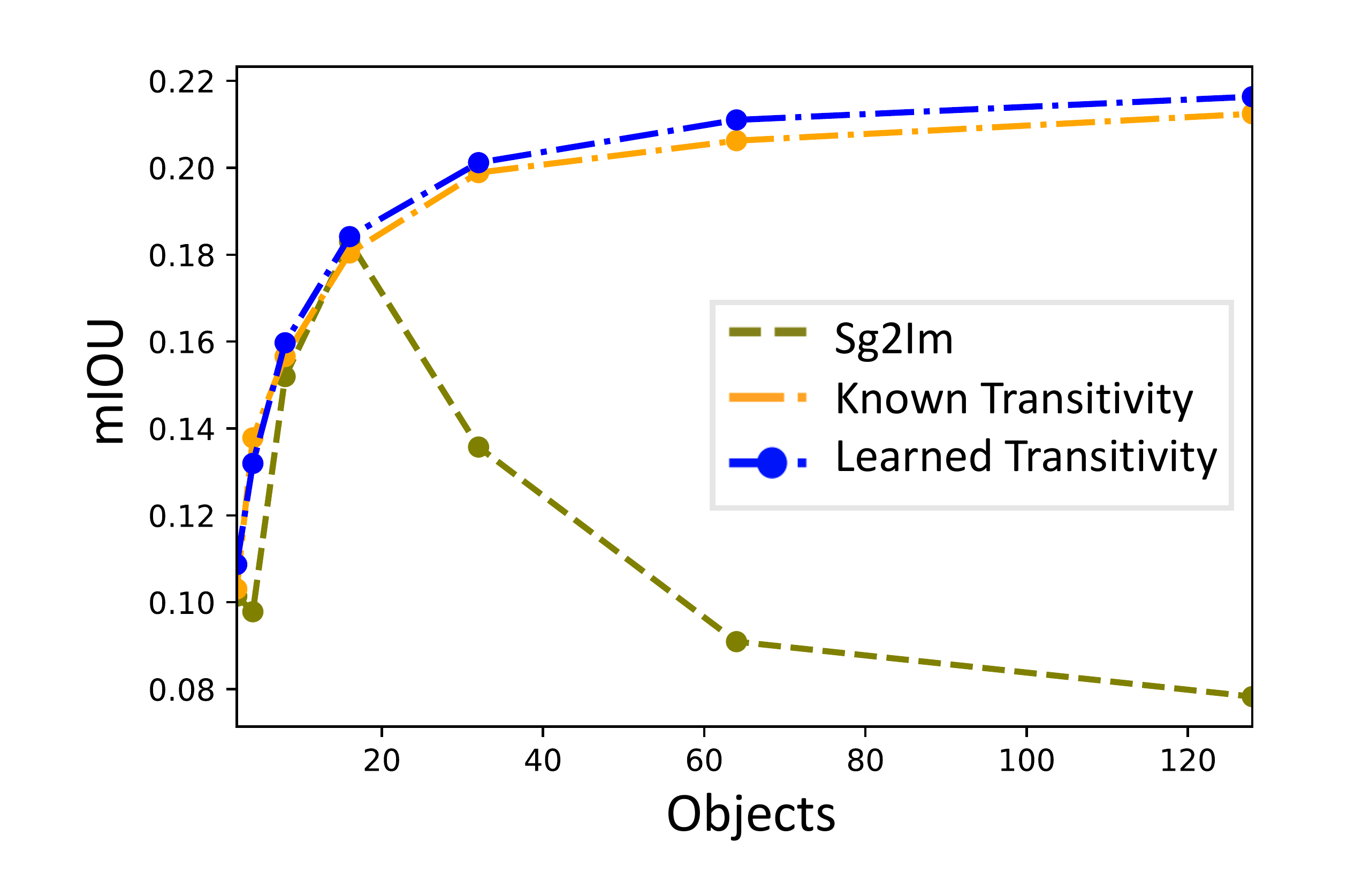}}
    \end{subfigure}
    \caption{\small{Synthetic dataset results. (a-b) The effect of the number of GCN layers on accuracy. Curves denote IOU performance as a function of the number of objects. Each point is a model trained and tested on a fixed number of objects given by the $x$ axis. (c) Out of sample number of objects. The model is trained on $16$ objects and evaluated on up to $128$ objects.}}
    \label{fig:syntheticone}
\end{figure*}

We train these models with two and four $GCN$ layers for up to $32$ objects. Additionally, to evaluate generalization to a different number of objects at test time, we train models with eight $GCN$ layers on 16 objects and test on up to $128$ objects. Results are shown in \figref{fig:syntheticone}a-b. First, it can be seen that the baseline performs significantly worse than transitivity based models. Second, ``Learned Transitivity'' closely matches ``Known Transitivity'' indicating that the model successfully learned which relations are transitive (we also manually confirmed this by inspecting $\theta^{trans}$). Third, the baseline model requires more layers to correctly capture scenes with more objects, whereas our model performs well with two layers. This suggests that WSGC indeed improves generalization ability by capturing invariances. \figref{fig:syntheticone}c shows that our model also generalizes well when evaluated on a much larger set of objects than what it has seen at training time, whereas the accuracy of the baseline severely degrades in this case.
\newline
{\bf Layout Accuracy on Packed Scenes.} Layout generation is particularly challenging in packed scenes. To quantify this, we evaluate on the Packed COCO and VG datasets. Since Sg2Im~\cite{johnson2018image}, PasteGAN~\cite{li2019pastegan}, and Grid2Im~\cite{ashual2019specifying} use the same SG-to-layout module, we compare WSGC only to Sg2Im~\cite{johnson2018image}. We test Sg2Im with 5,8 and 16 GCN layers to test the effect of model capacity.
The Packed setting in Table~\ref{tab:standard} shows that WSGC improves layout on all metrics. 



We also evaluate on the ``standard'' COCO/VG setting, which contain relatively few objects, and we therefore do not expect WSGC to improve there. Results in Table~\ref{tab:standard} show comparable performance to the baselines. In addition, manual inspection revealed that the learned $p^{conv}$ and $p^{trans}$ are overall aligned with expected values (See Supplementary). Finally, the results in the standard setting also show that increasing GCN size for Sg2Im~\cite{johnson2018image} results in overfitting.
\begin{table}[t!]
  \centering
  \setlength{\tabcolsep}{0.75mm}
  \scalebox{0.8}{
\footnotesize
    \begin{tabular}{c|cccccc|cccccc}
    \hline
       & \multicolumn{6}{c|}{Standard} & \multicolumn{6}{c}{Packed} \\ \cmidrule(l{0pt}r{0pt}){2-7} \cmidrule(l{0pt}r{0pt}){8-13}
      Method & \multicolumn{2}{c|}{mIOU} & \multicolumn{2}{c|}{R@0.3} & \multicolumn{2}{c|}{R@0.5} & \multicolumn{2}{c|}{mIOU} & \multicolumn{2}{c|}{R@0.3} & \multicolumn{2}{c}{R@0.5} \\
      & COCO & VG & COCO & VG & COCO & VG & COCO & VG & COCO & VG & COCO & VG \\\hline \hline
      Sg2Im \cite{johnson2018image} 5 $GCN${\tablefootnote{Results copied from manuscript. \label{frompaper}}}&     
      {-} & {-} & {52.4} & {21.9} & {32.2} & {10.6} & {-} & {-} & {-} & {-} & {-} & {-} \\ 
      \hline
      Sg2Im \cite{johnson2018image} 5 $GCN${\tablefootnote{Our implementation of \cite{johnson2018image}. This is the same as our model without WSGC.\label{ours}}}
      & 41.7 & 16.9 & 62.6 & 24.7 & 37.5 & 9.7 & 35.8 & 25.4 & 56.0 & 36.2 & 25.3 & 15.8 \\
      
      Sg2Im \cite{johnson2018image} 8 $GCN^{\ref{ours}}$  
      & 41.5 & \textbf{18.3} & 62.9 & \textbf{26.2} & 38.1 & 10.6 & 37.2 & 25.8 & 58.6 & 36.9 & 26.4 & 15.9 \\

      Sg2Im \cite{johnson2018image} 16 $GCN^{\ref{ours}}$  
      & 40.8 & 16.4 & 61.4 & 23.3 & 36.6 & 7.8 & 37.7 & 27.1 & 60.3 & 39.0 & 26.6 & 17.0 \\
      \hline
    
      WSGC 5 $GCN$ (ours) & \textbf{41.9} & 18.0 & \textbf{63.3} & 25.9 & \textbf{38.2} & 10.6 & \textbf{39.3} & \textbf{28.5} & \textbf{62.6} & \textbf{42.4} & \textbf{30.1} & \textbf{18.3} \\
      \hline
    \end{tabular}
  }
  \caption{\small{
    Accuracy of predicted bounding boxes. We consider two different data settings: ``Standard'' and ``Packed''. (a) \textbf{Standard}: Training and evaluation is on VG images with 3 to 10 objects, and COCO images with 3 to 8 objects. (b) \textbf{Packed}: Training and evaluation is on images with 16 or more objects.}
  }
  \label{tab:standard}
\end{table}

\begin{table}[t!]
  \centering
  \setlength{\tabcolsep}{1mm}
  \scalebox{0.9}{
    \begin{tabular}{c|ccc|ccc}
    \hline
      \multirow{2}{*}{Method} & \multicolumn{3}{|c|}{Semantically Equivalent} & \multicolumn{3}{c}{Noisy SGs} \\
      \cmidrule(l{0pt}r{0pt}){2-4} \cmidrule(l{0pt}r{0pt}){5-7}
      & \multicolumn{1}{c|}{mIOU} & \multicolumn{1}{c|}{R@0.3} & \multicolumn{1}{c|}{R@0.5} & \multicolumn{1}{c|}{mIOU} & \multicolumn{1}{c|}{R@0.3} & \multicolumn{1}{c}{R@0.5} \\

      \hline \hline
      Sg2Im \cite{johnson2018image} 5 $GCN^{\ref{ours}}$ & 21.8 & 29.5 & 10.7 & 29.4 & 42.9 & 17.8   \\
      Sg2Im \cite{johnson2018image} 8 $GCN^{\ref{ours}}$   & 23.6 & 33.2 & 11.4 & 29.9 & 43.7 & 18.8   \\
      Sg2Im \cite{johnson2018image} 16 $GCN^{\ref{ours}}$  & 21.6 & 29.0 & 10.1 & 28.7 & 41.8 & 17.7 \\ 
      \hline
      WSGC 5 $GCN$ (ours) & \textbf{35.3} & \textbf{53.2} & \textbf{25.7} & \textbf{31.8} & \textbf{46.6} & \textbf{21.9} \\
      \hline
    \end{tabular}}
 \caption{Evaluating the robustness of the learned canonical representation for models which were trained on Packed COCO. For each SG, a \textbf{semantically equivalent} SG is sampled and evaluated at test time. Additionally, models are evaluated on \textbf{Noisy SGs}, for which edges contain $10\%$ randomly chosen relations.
  }
  \label{tab:canonical}
\end{table}

{\bf Generalization on Semantically Equivalent Graphs.} A key advantage of WSGC is that it produces similar layouts for semantically equivalent graphs. This is not true for methods that do not use canonicalization. To test the effectiveness of this property, we modify the test set such that input SGs are replaced with semantically equivalent variations. For example if the original SG was $(A,\mbox{right},B)$ we may change it to $(B,\mbox{left},A)$. To achieve this, we generate a semantically equivalent SG by randomly choosing to include or exclude edges which do not change the semantics of the SG. We evaluate on the Packed COCO dataset. Results are shown in Table \ref{tab:canonical} and qualitative examples are shown in Figure~\ref{fig:canonicalization}. It can be seen that WSGC significantly outperforms the baselines.

{\bf Testing Robustness to Input SGs}. Here we ask what happens when input SGs are modified by adding ``noisy'' edges. This could happen due to noise in the annotation process or even adversarial modifications. Ideally, we would like the generation model to be robust to small SG noise. We next analyze how such modifications affect the model by randomly modifying $10\%$ of the relations in the COCO data. As can be seen in Table~\ref{tab:canonical}, the WSGC model can better handle  noisy SGs than the baseline. We further note that our model achieves good results on the VG dataset, which was manually annotated, suggesting it is robust to annotation noise.
 The results in Table~\ref{tab:canonical} also show the Sg2Im generalization deteriorates when growing from 8 to 16 layers, suggesting that the effect of canonicalization cannot be achieved by just increasing model complexity.

\begin{figure*}[t!]
    \centering
    \includegraphics[width=\linewidth]{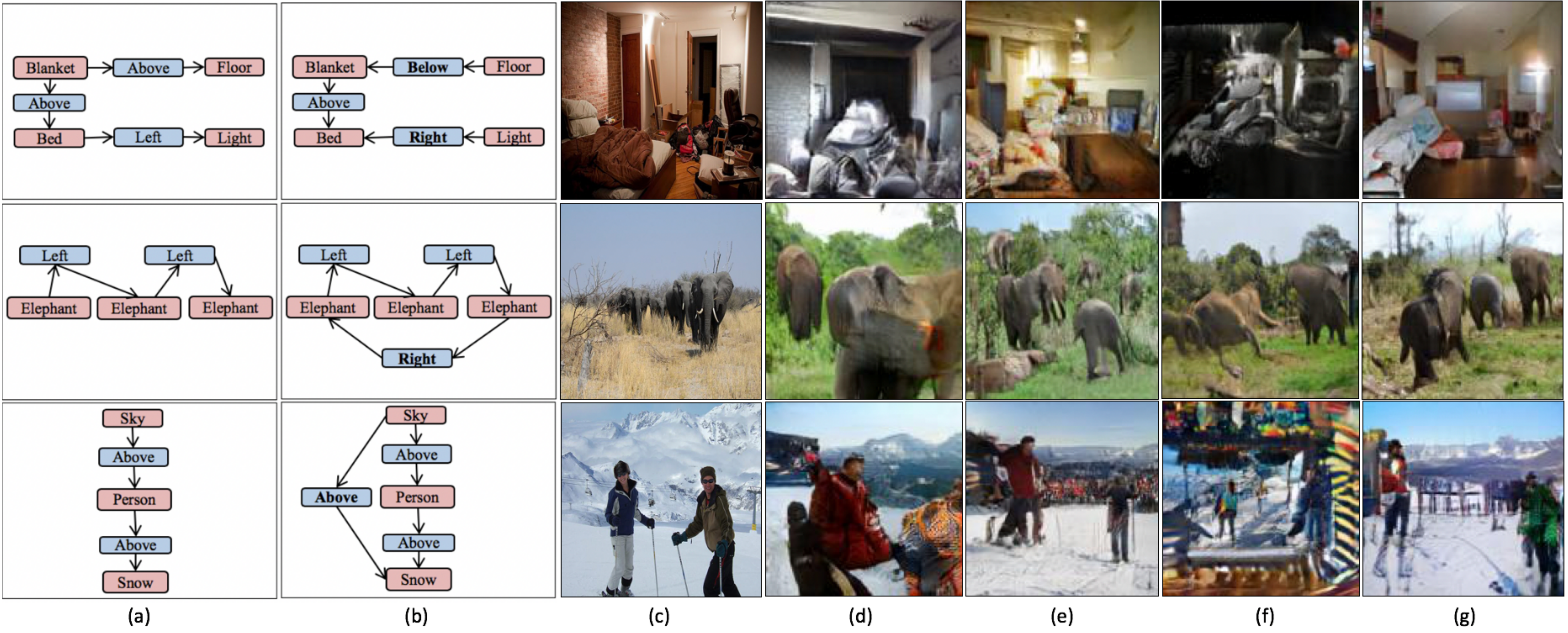}
    \caption{\small{Generalization from Semantically Equivalent Graphs. Each input SG is changed to a semantically equivalent SG at test time. The layout-to-image model is LostGAN~\cite{Sun_2019_ICCV} and different SG-to-layout models are tested. (a) Original SG (partial). (b) A  modified semantically equivalent SG (partial). (c) GT image. (d-e) Sg2Im~\cite{johnson2018image} and WSGC for the original SG. (f-g) Sg2Im~\cite{johnson2018image} and WSGC for the modified SG.}}
    \label{fig:canonicalization}
\end{figure*}

\begin{table}[t!]
\small{
  \setlength{\tabcolsep}{2mm}
  \centering{
  \begin{tabular}{c|ccc}
  \hline
    \multirow{2}{*}{Method} & \multicolumn{2}{c|}{Inception} & \multicolumn{1}{c}{Human} \\
              \cmidrule(l{0pt}r{0pt}){2-4}
    & \multicolumn{1}{c|}{COCO}  & \multicolumn{1}{c|}{VG} 
    & \multicolumn{1}{c}{CLEVR} \\
    \hline\hline               
     Sg2Im~\cite{johnson2018image} & $5.4 \pm 0.3$ & $7.6 \pm 1.0$ & $3.2\%$ \\
     WSGC (ours) & \textbf{5.6 $\pm$ 0.1}  & \textbf{8.0 $\pm$ 1.1} & $96.8\%$ \\
     \hline
     GT Layout &  5.5 $\pm$ 0.4 & 8.2 $\pm$ 1.0 & - \\
     \hline
  \end{tabular}}
  \caption{Results for SG-to-image on \textbf{Packed} datasets ($16+$ objects). For VG and COCO we use the layout-to-image architecture of \textbf{LostGAN}~\cite{Sun_2019_ICCV} and test the effect of different SG-to-layout models. For CLEVR, we use our \textbf{AttSPADE} generator.}
  \label{tab:genall}
  }
 \end{table}

\subsection{Scene-graph-to-image Generation}
To test the contribution of our proposed Scene-Graph-to-layout approach to the overall task of SG-to-image generation, we further test it in an end-to-end pipeline for generating images. For Packed COCO and Packed VG, we compare our proposed approach with  Sg2Im~\cite{johnson2018image} using a fixed pre-trained LostGAN \cite{sun2019image} as the layout-to-image generator. For CLEVR, we use WSGC and our own AttSPADE generator (see \secref{sec:sg_to_image}). We trained the model on images with a maximum of 10 objects and tested on larger scenes with 16+ objects.

We evaluate performance using Inception score \cite{salimans2016improved} and a study where Amazon Mechanical Turk raters were asked to rank the quality of two images: one generated using our layouts, and the other using SG2Im layouts.\footnote{We used raters only for the CLEVR data, where no GT images or bounding boxes are available for 16+ objects, and thus Inception cannot be evaluated.} Results are provided in Table \ref{tab:genall}. For COCO and VG it can be seen that WSGC improves the overall quality of generated images. In CLEVR, Table \ref{tab:genall}, WSGC outperforms Sg2Im in terms of IOU. {{In ${96.8\%}$ of the cases, our generated images were ranked higher than SG2Im}}. Finally, Figures \ref{fig:clevr_ashes} and \ref{fig:genresults} provide qualitative examples and comparisons of images generated based on CLEVR and COCO. More generation results on COCO and VG can be seen in the Supplementary.

\begin{figure*}[t!]
    \centering
    \includegraphics[width=\linewidth]{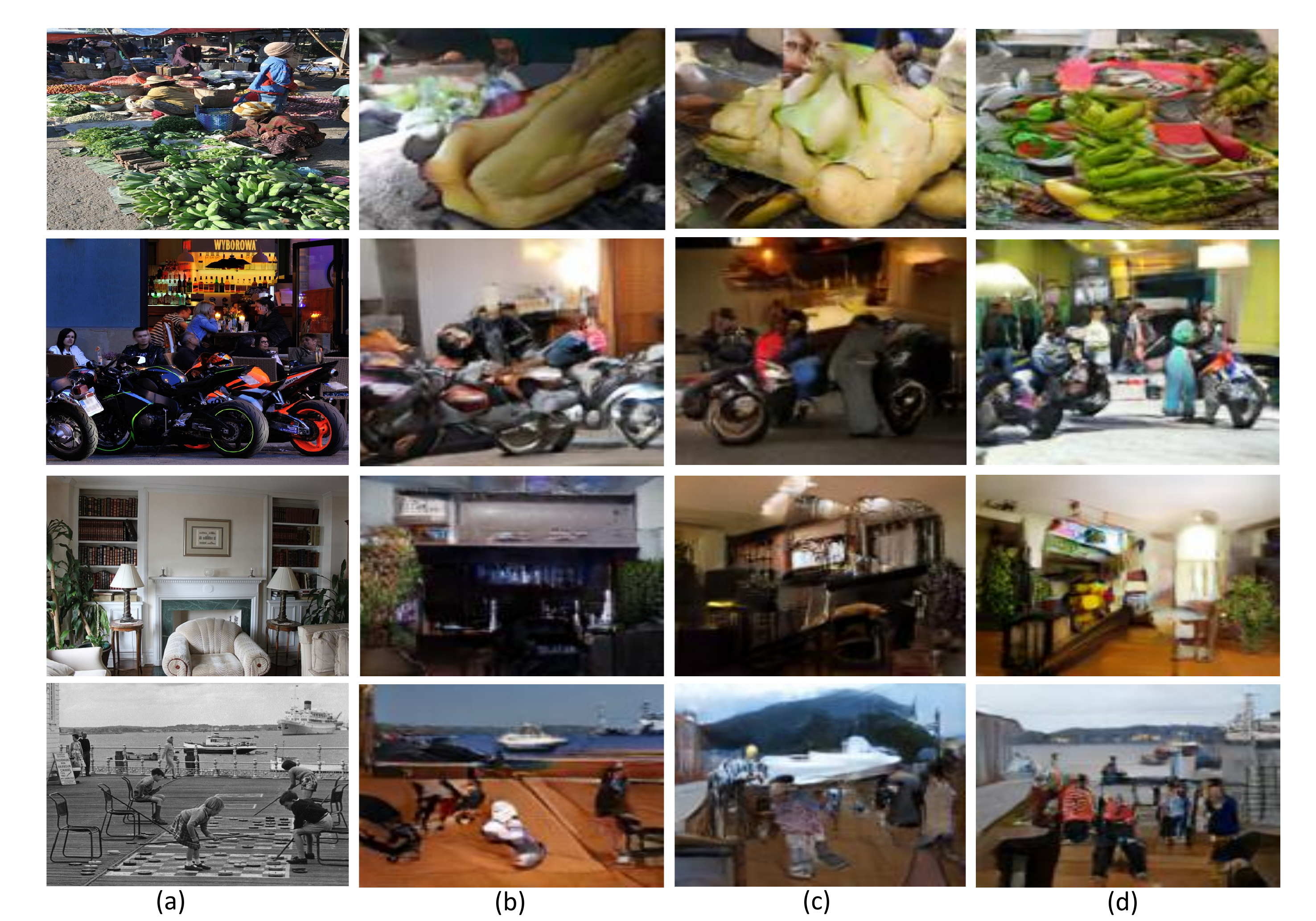}
    \caption{Selected Scene-graph-to-image generation results on the Packed-COCO dataset. Here, we fix the layout-to-image model to LostGAN~\cite{Sun_2019_ICCV}, while changing different scene graph-to-layout models. (a) GT image. (b) Generation from GT layout. (c) Sg2Im~\cite{johnson2018image} model with LostGAN~\cite{Sun_2019_ICCV}. (d) Our WSGC model with LostGAN~\cite{Sun_2019_ICCV}.}
    \label{fig:genresults}
\end{figure*}

\section{Conclusion}
\label{sec:conclude}
We presented a method for mapping SGs to images that is invariant to a set of logical equivalences. Our experiments show that the method results in improved layouts and image quality. We also observe that canonical representations allow one to handle packed scenes with fewer layers than non-canonical approaches. Intuitively, this is because the closure calculation effectively propagates information across the graph, and thus saves the need for propagation using neural architectures. The advantage is that this step is hard-coded and not learned, thus reducing the size of the model. Our results show the advantage of preprocessing an SG before layout generation. Here we studied this in the context of two types of relation properties. However, it can be extended to more complex ones. In this case, finding the closure will be computationally hard, and would amount to performing inference in Markov Logic Networks \cite{richardson2006markov}. On the other hand, it is likely that modeling such invariances will result in further robustness of the learned models, and is thus an interesting  direction for future work.

\section*{Acknowledgements}
\label{sec:ack}
This project has received funding from the European Research Council (ERC) under the European Unions Horizon 2020 research and innovation programme (grant ERC HOLI 819080). Prof. Darrell's group was supported in part by DoD, NSF, BAIR, and BDD. This work was completed in partial fulfillment for the Ph.D degree of the first author.


\bibliographystyle{splncs04}
\bibliography{egbib}

\newpage
\setcounter{section}{0} 

In this supplementary file we provide additional implementation details, empirical results, and a proof of correctness for the SGC algorithm. 
\section{Scene-Graph-to-Layout} 
\label{sec:supp:sg-to-layout}
In Section 3.2, we introduced the WSGC-E and WSGC-S methods, two different procedures proposed for mapping an input scene graph into a weighted canonicalized relation graph. As mentioned in Section 3.2, although the WSGC-E is a natural extension of the SGC procedure, it is impractical for large complex graphs whereas the WSGC-S method adds fewer edges and is thus more practical for training. In what follows, we provide additional details about WSGC-E and WSGC-S, as well as comparison and analysis.

\subsection{Exact Weighted Scene Graph Canonicalization (WSGC-E)} 
Next, we describe in detail the WSGC-E method for obtaining a weighted relation graph that is a natural extension of \hardmethod. The \softslowmethod{} begins with the user-specified graph $E$, with weights of one. Next two weighted completion steps are performed, corresponding to the {\hardmethod} steps.

\noindent {\bf Converse Completion:} In SGC, this step adds all converse edges. In the weighted case it makes sense to add the converse edge with its corresponding converse weight.
For example, if the graph $E$ contains the edge $(i,\text{above},j,1)$ and $p^{conv}(\text{below}|\text{above}) =0.7$, we add the edge $(j,\text{below},i,0.7)$. See \figref{fig:WSGC2}b.

\noindent {\bf Transitive Completion:} In SGC, all transitive edges are found and added. In the weighted case, a natural alternative is to set a weight of a path to be the product of weights along this path, and set the weight of a completed edge $(i,r,j)$ to be the maximum weight of a path between $i$ and $j$ times the probability $p^{trans}(r)$  that the relation is transitive. See \figref{fig:WSGC2}c. The maximum path weight problem is equivalent to maximizing the sum of log probabilities, and since these are all negative, this can be solved in polynomial time via a shortest weight path algorithm (e.g., FW). However, when there are many nodes and relations, runtime can still be substantial, and thus we offer a faster approach next.

{\bf Training with WSGC-E.} In the main text, we described the training loss and optimization for WSGC-S. Optimizing the loss for WSGC-E is similar, as we explain next. We describe the loss of the WSGC-E method for a single input scene graph $E$ and its ground truth layout $Y$. 
The parameters of the model are as follows: $\theta^{g}$ are the parameters of the GCN in Section 4, $\theta^{trans}$ are the parameters of the transitive probability (Eq. 1), and $\theta^{conv}$ are those of the converse probability (Eq. 2). Let $\theta$ denote the set of all parameters. Denote the GCN applied to this graph by $G_{\theta^g}$.
We use $L_1$ as the loss between predicted and ground-truth bounding boxes $Y$. Namely, we wish to minimize the following objective (we write WSGC instead of WSGC-E below for brevity):
\begin{equation}
L(\theta) = \normm{Y -  G_{\theta^{g}}(\text{WSGC}(E;\theta^{trans},\theta^{conv}))}_1 \quad
\label{eq:wsgce_loss}
\end{equation}
When calculating $L(\theta)$, most of the operations are standard and are differentiated automatically by PyTorch. The only apparent complication is with the minimum weight path. However, we next explain why there is actually not a problem and one can simply take the gradient of the PyTorch computation graph for $L(\theta)$, which includes the minimum-weight-path computation. 

Recall that in WSGC-E we first weight each edge by the corresponding converse weights $p^{conv}_{\theta}$. Let $w(e;\theta)$ denote the weight of edge $e$ after this weighting step. Next, we perform a transitive completion as follows. Given an edge $e'$, its new weight will be (up to the multiplicative factor of $p^{trans}$ which we leave out for brevity):  
\begin{equation}
    w_{trans}(e';\theta) = \max_{P\in\mathcal{P}}\prod_{e\in P} w(e;\theta) =  e^{\max_{P\in\mathcal{P}}\sum_{e\in P} \log{w(e;\theta)}}
\end{equation}
where $\mathcal{P}$ are all the paths between the two incident nodes of $e'$. To calculate the sub-gradient of this expression with respect to $\theta$, we note that in the exponent we have a maximum over linear functions, and thus differentiating it wrt $\theta$ corresponds to finding the maximizing $P*$ and then differentiating $\sum_{e\in P*} \log{w(e;\theta)}$.

\begin{figure*}[t!]
  \centering
    \includegraphics[width=\linewidth]{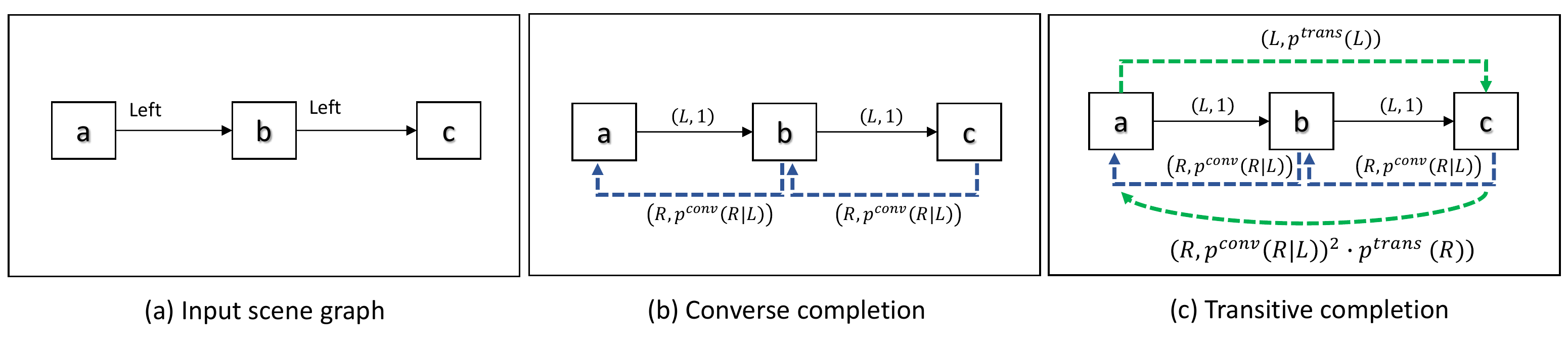}
    \caption{\small{An illustration of WSGC-E where relations are Left (L) and Right (R). (a) The input graph contains two relations with weight $1$. (b) Converse edges (blue dashed arrows) are completed with the weights $p^{conv}$. (c) Transitive edges (green dashed arrows) are added and assigned the weight of the corresponding path times $p^{trans}$.}}
    \label{fig:WSGC2}
\end{figure*}

Technically, the above suggests a very simple PyTorch implementation. Implement the computation graph for \eqref{eq:wsgce_loss}, including the non-differentiable maximum weight path computation. And then let PyTorch take gradients for this graph. Since the maximum-weight-path cannot be differentiated through, the computation will fix the maximizing path and take the gradient there, and this is indeed the correct sub-gradient, as per our discussion above.

The downside of the WSGC-E method is that it assigns weights to all edges in the graph, and for all relations, and thus computations involve a dense graph, which makes training and inference slow. This motivates our use of WSGC-S which uses sparser graphs.

\label{sec:supp:sg-to-layout:wsgce}

\subsection{Empirical Comparison of WSGC-E and WSGC-S}

Table \ref{supp:tab:WSGCE_WSGCS} shows a comparison of WSGC-E and WSGC-S on the standard COCO and VG datasets, where WSGC-E runs in a reasonable time so that comparison is possible. The size of the graphs on the standard datasets is less than an average of $1000$ triplets per image, while on the packed datasets it is $24,000$ triplets per image. Thus it is impossible to run the WSGC-E on packed datasets. It can be seen that the methods achieve comparable performance, suggesting that indeed WSGC-S is a scalable alternative to WSGC-E.




\begin{table}
  \centering
  \setlength{\tabcolsep}{1mm}
  \scalebox{1.0}{
    \begin{tabular}{c|ccc|ccc}
    \hline
      \multirow{2}{*}{Method} & \multicolumn{3}{|c|}{COCO} & \multicolumn{3}{c}{Visual Genome} \\
      \cmidrule(l{0pt}r{0pt}){2-4} \cmidrule(l{0pt}r{0pt}){5-7}
      & \multicolumn{1}{c|}{mIOU} & \multicolumn{1}{c|}{R@0.3} & \multicolumn{1}{c|}{R@0.5} & \multicolumn{1}{c|}{mIOU} & \multicolumn{1}{c|}{R@0.3} & \multicolumn{1}{c}{R@0.5} \\
      \hline \hline
      WSGC-S 5 $GCN$ & 41.9 & \textbf{63.3} & 38.2 & 18.0 & 25.9 & \textbf{10.6}  \\
      \hline
      WSGC-E 5 $GCN$ & \textbf{42.2} & 63.0 & \textbf{38.7} & 18.0 & \textbf{26.5} & 9.9\\
      \hline
    \end{tabular}}
    \caption{Evaluation of WSGC-E and WSGC-S on Standard COCO and Visual Genome.}
  \label{supp:tab:WSGCE_WSGCS}

\end{table}

\subsection{Analysis of Learned Weights \label{supp:sg-to-layout:weights}}
Our approach parameterizes the weights of converse and transitive relations and learns these parameters from data. It is interesting to see whether the learned weights recover known converse and transitive relations. 

Inspecting the converse weights $p^{conv}$ that were learned on the standard COCO dataset reveals that all weights have converged to values close to $0$ and $1$, and align well with the expected true converse relation. Specifically, weights corresponding to converse pairs such as (``below'', ``above'') all converged to $1$, while the rest of the pairs, such as (``left of'', ``inside'') converged to $0$. For transitive weights $p^{trans}$, $5/6$ of the transitive relations correctly converged to $1$ and a single relation to $0$. Concretely, ``above'', ``left of'', ``right of'', ``inside'' and ``below'' converged to $1$ while ``surrounding'' did not. The learned values are shown in Figure~\ref{fig:supp:coco_weights}.

\begin{figure*}[t!]
    \centering
    \includegraphics[width=\linewidth]{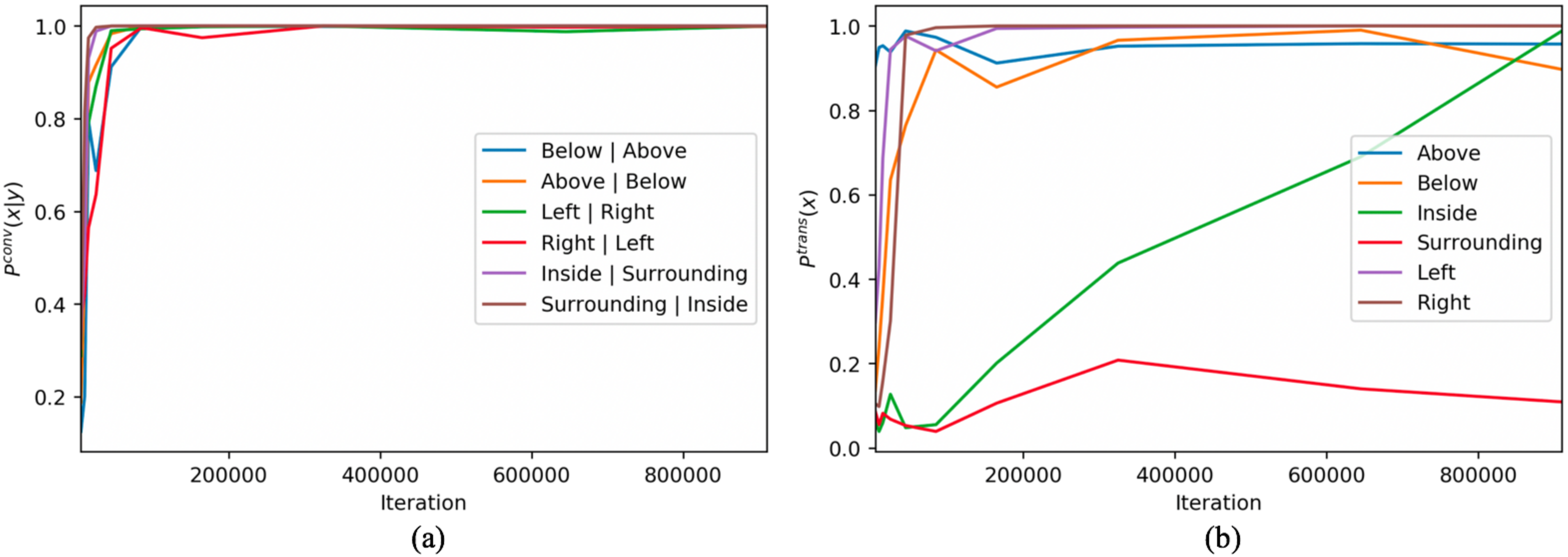}
    \caption{Learned $p^{conv}$ and $p^{trans}$ weights for the WSGC-S model on the COCO dataset. The learned values of $p^{conv}$ (see a) and of $p^{trans}$ (see b) are presented as function of training iteration.}
    \label{fig:supp:coco_weights}
\end{figure*}

\subsection{Weighted Graph Convolutional Network}
\label{sec:supp:sg-to-layout:gcn}
In Section 4, we presented Graph Convolutional Network (GCN)~\cite{kipf2016semi} as a natural architecture for the SG to layout task. We use a similar approach to recent methods~\cite{ashual2019specifying,johnson2018image} for this task, but modify the GCN to our weighted scene graph. This is done by revising the graph convolution layer such that the aggregation step of each node is set to be a weighted average, where the weights are those in the canonical SG. In what follows, we provide additional details about our Weighted GCN.

Each object category $c\in\cC$ is assigned a learned embedding $\phiv_c\in\reals^D$ and each relation $r\in\cR$ is assigned a learned embedding $\psiv_r\in\reals^D$. Given an SG with $N$ objects, the GCN iteratively calculates a representation for each object and each relation in the graph. 
Let $\vv^{k}_i\in\reals^d$ be the representation of the $i^{th}$ object in the $k^{th}$ layer of the GCN. Similarly, for each edge $e=(i,r,j,w)$ in the graph let $\uu^k_e\in\reals^d$ be the representation of the relation in this edge. These representations are calculated as follows. Initially we set:
$\vv^0_i = \phiv_{o(i)}, \uu^0_e = \psiv_{r(e)}$,
where $r(e)$ is the relation for edge $e$. Next, we use three functions (MLPs) $F_s,F_r,F_o$, each from $\reals^D\times\reals^D\times\reals^D$ to $\reals^D$ to obtain an updated object representation (see Section~\ref{sec:supp:impl_details:sg-to-layout} for implementation details). These can be thought of as processing three vectors on an edge (the subject, relation and object representations) and returning three new representations. Given these functions, the updated object representation is the weighted average of all edges incident on $i$:\footnote{Note that a box can appear both as a ``subject'' and an ``object'' thus two different sums in the denominator and the normalization is needed because we want to obtain a new single object representation while the number of object occurrences is varied.}
\begin{equation}
    \vv_i^{t+1} = \frac{1}{c}\left[\sum_{e=(i,r,j,w)} w F_s(\vv_i^t,\uu_e^t,\vv_j^t)+  \sum_{e=(j,r,i,w)} w F_o(\vv_j^t,\uu_e^t,\vv_i^t)\right]    
\label{eq:supp:eq1}
\end{equation}
where $c$ is a normalizing constant $c=\sum_{e=(i,r,j,w)}w+\sum_{e=(j,r,i,w)}w$.
For the edge we set: $\uu_e^{t+1} = F_r(\vv_i^{t+1},\uu_e^t,\vv_j^{t+1})$. 

After iterating the GCN for $L$ updates, the layout for node $i$ is obtained by applying an MLP with four outputs to $\vv_i^L$.\footnote{The MLP has a sigmoid activation in the last layer so that the predicted normalized bounding box coordinates are in $[0,1]$.} Note that $F_s,F_r,F_o$ and $w$ depend on learned parameters which are optimized using gradient descent.

\subsection{Generalization on Packed Scenes}
\label{sec:supp:sg-to-layout:packed_scenes}

To further test the effect of model capacity from Table 1 in the paper, we even trained bigger Sg2Im models with $32,64$ layers on Packed COCO, resulting in IOU of $36.93, 11.65$. We also trained a Sg2Im model with $1024$ hidden units and $16$ layers, and IOU deteriorated to $37.01$. These results suggest that increasing the capacity of Sg2Im leads to overfitting and that WSGC improvement is indeed due to canonicalization.

\section{Layout-to-Image with AttSPADE}
\label{sec:supp:attspade}

\begin{figure*}[t!]
    \centering
    \includegraphics[width=\linewidth]{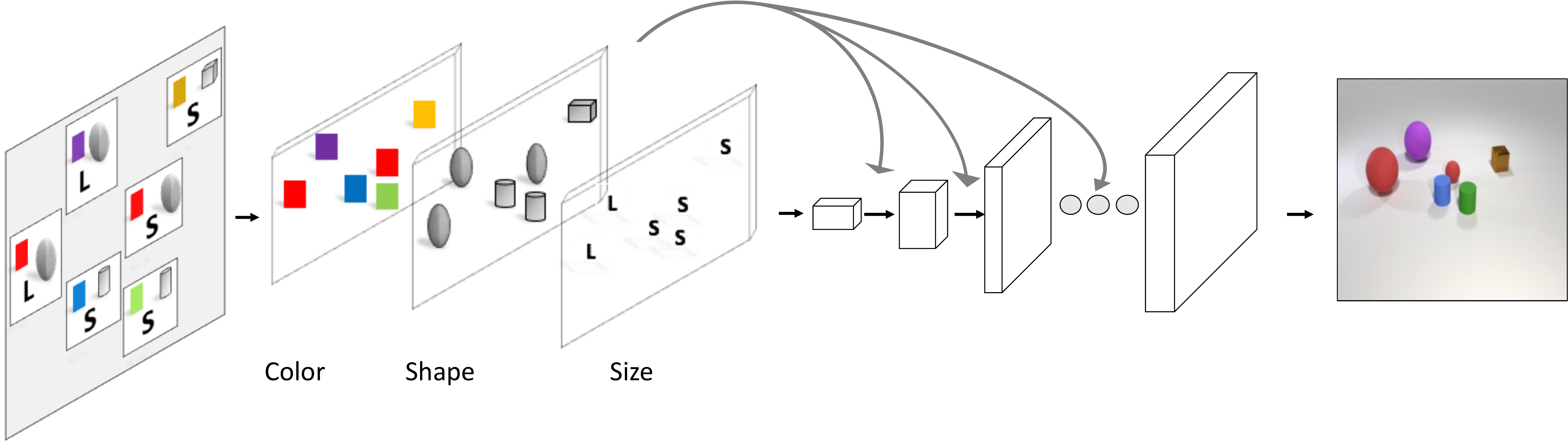}
    \caption{Generating images with our AttSPADE model. Given a layout of boxes, our model generates an image using the layout into a series of residual blocks with upsampling layers. The layout is modeled by multiple semantic attributes per box rather than a single class descriptor.}
    \label{fig:supp:attspade}
\end{figure*}

For the CLEVR dataset \cite{johnson2017clevr}, we use a novel generator, which we refer to as AttSPADE. This generator can be used for directly controlling attributes of the generated image, and this is not supported by other generators such as LostGAN. Although the generator is not the main focus of our contribution, we believe it is of independent interest, and thus describe it in some detail below, and show images that it generates.

\subsection{The AttSPADE Model}
The key idea in the AttSPADE model is to condition generation on the attribues, as opposed to only the object class as done in current models. In what follows we describe the model.

We consider the case where a bounding box has an associated set of attributes. For example, the object category is an attribute and the size is an attribute (with possible values ``small'', ``medium'' and ``large''). Additionally, if a segmentation mask is provided as input, it can be added as a binary attribute. We encode this set of attributes via a multi-hot vector $\zz\in\reals^r$ that is set to one for the corresponding attributes, and apply a FC layer to it to obtain a vector $\vv\in\reals^d$. Next, we construct a tensor $M \in \mathbb{R}^{d \times H \times W}$ where $H$ and $W$ are the boxes height and width and $M[:,i,j]=\vv$. This encodes the attributes for each pixel in the bounding box.\footnote{We note that different pixels may have different attributes in principle although we don't use this here.} Finally, we use $M$ as input to a SPADE~\cite{park2019semantic} generator to obtain the generated image. Thus, our approach simply replaces the input of the SPADE model (which is just an object mask) with the tensor $M$.

Lastly, our model uses two discriminators: one for the image (to achieve a better quality of the entire image), and one for the boxes (in order to better capture each box). This is similar to~\cite{ashual2019specifying,johnson2018image} but with a few modifications (see next section). A high level description of the architecture is shown in \figref{fig:supp:attspade}.

\subsection{The Loss Functions}
\label{sec:supp:attspade:loss}

Our AttSPADE model contains several modifications of the loss functions. First, the generator is trained with the same multi-scale discriminator and loss function used in pix2pixHD~\cite{wang2018pix2pixHD}, except we replace the squared error loss \cite{mao2017least} with the hinge loss \cite{lim2017geometric,miyato2018spectral,zhang2018self}. Second, since our layout-to-image model generates the image from a given layout of bounding boxes, we add a box term loss to guarantee that the generated objects in these boxes look real. For this purpose, we crop the bounding boxes to create object images and train the discriminator to discriminate between real object images and generated object images. The image discriminator is implemented as in SPADE~\cite{park2019semantic}.

\subsection{Baseline Models}
\label{sec:supp:attspade:baselines}
We report generation results that vary both the layout being used and the layout-to-image component. For the layout we consider three options: (1) Ground truth layout. (2) Our WSGC predicted layout. (3) The layout used in Sg2Im~\cite{johnson2018image}. For the image generation we use three options: (1) Our AttSpade generator. (2) The LostGAN generator \cite{Sun_2019_ICCV} (the most recent state-of-the-art generation model). (3) The Grid2Im \cite{ashual2019specifying} generator, which uses the same graph model as~\cite{johnson2018image}. The results reported in \cite{ashual2019specifying} use a coarse version of the GT layout (i.e., the layout rounded to a $5\times5$ grid). Since this variant comes close to actually using the GT layout, we also consider an additional version of \cite{ashual2019specifying} that does not use this information. We refer to this version as ``Grid2Im No-Att'' (code provided by the authors of \cite{ashual2019specifying}).

For a fair comparison, all models were tested  with the same external code evaluation metrics.

\begin{table*}[t!]
    \centering
    \resizebox{\textwidth}{!}{
    \begin{tabular}{c|cc|cc|cc|cc}
    \multirow{2}{*}{\textbf{Resolution}} & \multicolumn{2}{c}{\textbf{Methods}} & \multicolumn{2}{c}{\textbf{Inception Score}} & \multicolumn{2}{|c}{\textbf{FID}} & \multicolumn{2}{|c}{\textbf{Diversity Score}} \\
    & \textbf{SG-to-Layout} & \textbf{Layout-to-Image} & \textbf{COCO} & \textbf{VG} & \textbf{COCO} & \textbf{VG} & \textbf{COCO} & \textbf{VG}\\ 
    \hline

    \multirow{6}{*}{128x128} & Real Images & Real Images & 23.0 $\pm$ 0.4 & 22.8 $\pm$ 1.7 & - & - & - & - \\
    \cline{2-9}
    & GT Layout & Grid2Im~\cite{ashual2019specifying} &  12.5 $\pm$ 0.3 & - & 59.5 & - & - & - \\
    & GT Layout & LostGAN~\cite{Sun_2019_ICCV} & 11.8 $\pm$ 0.3 & 8.9 $\pm$ 0.3 & 64.0 & 66.7 & \textbf{0.57 $\pm$ 0.06} & \textbf{0.59 $\pm$ 0.06} \\ 
    & GT Layout & AttSPADE (Ours) &\textbf{15.6 $\pm$ 0.5} & \textbf{11.7 $\pm$ 0.8} & \textbf{54.7} & \textbf{36.4} & 0.44 $\pm$ 0.09 & 0.51 $\pm$ 0.08 \\ 
    
    \cline{2-9}
    & WSGC & LostGAN~\cite{Sun_2019_ICCV} & \textbf{11.1 $\pm$ 0.6} & 8.1 $\pm$ 0.3 & \textbf{65.9} & 73.4 & 0.57 $\pm$ 0.06 & 0.58 $\pm$ 0.06 \\
    & Sg2Im~\cite{johnson2018image} & Grid2Im~\cite{ashual2019specifying} &10.4 $\pm$ 0.4 & - & 75.4 & - & - & - \\
    & WSGC & AttSPADE (Ours) & 10.8 $\pm$ 0.5 & \textbf{10.0 $\pm$ 0.7} & 73.8 & \textbf{46.4} & \textbf{0.57 $\pm$ 0.06} & \textbf{0.58 $\pm$ 0.06} \\ 
    \hline
    
    \multirow{6}{*}{256x256} & Real Images & Real Images & 30.3 $\pm$ 1.4 & 31.7 $\pm$ 2.0 & - & - & - & - \\
    \cline{2-9}
    & GT Layout & Grid2Im~\cite{ashual2019specifying} & 16.4 $\pm$ 0.7 & - & 65.2 & - & 0.48 $\pm$ 0.09 & - \\
    & GT Layout & AttSPADE (Ours) & \textbf{19.5 $\pm$ 0.9} & \textbf{16.9 $\pm$ 1.2} & \textbf{64.65} & \textbf{42.9} & \textbf{0.55 $\pm$ 0.11} & \textbf{0.62 $\pm$ 0.08} \\
    
    \cline{2-9}
    & Sg2Im~\cite{johnson2018image} & Grid2Im~\cite{ashual2019specifying} No-Att & 6.6 $\pm$ 0.3 & - & 127.0 & - & 0.65 $\pm$ 0.05 & - \\
    & WSGC & AttSPADE (Ours) &\textbf{13.9 $\pm$ 0.3} & \textbf{16.5 $\pm$ 0.7} & \textbf{119.1} & \textbf{45.7} & \textbf{0.70 $\pm$ 0.07} & \textbf{0.68 $\pm$ 0.07} \\
    
    \end{tabular}}
    \caption{Quantitative comparisons for SG-to-image methods using Inception Score (higher is better), FID (lower is better) and Diversity Score (higher is better). Evaluation is done on the COCO-Stuff and VG datasets.}
    \label{tab:supp:gen_results} 
\end{table*}

\subsection{Results}
\label{sec:supp:attspade:results}
The results in Table \ref{tab:supp:gen_results} suggest that the AttSPADE model improves over previous approaches~\cite{ashual2019specifying,Sun_2019_ICCV} when generating an image from a GT layout, in both resolutions. In addition, our end-to-end model, which includes the WSGC and AttSPADE model, outperforms most of the baselines on the COCO and Visual Genome datasets. 

\figref{fig:supp:coco128} shows a direct comparison between different generators using GT layout for COCO. It can be seen that AttSPADE provides higher quality images than the other generators.

\figref{fig:supp:coco256} shows different generators that use both GT and generated layouts for COCO. Additional qualitative results on Visual Genome can be seen in \figref{fig:supp:vg128_compare1} and \figref{fig:supp:vg256_compare1}. In the generation results it can be seen that AttSPADE produces more realistic images, when compared to other generators. Furthermore when using WSGC layout the images are qualitatively similar to using GT layout, which suggests that WSGC produces high quality layouts.

\section{Datasets}
\label{sec:supp:datasets}

\subsection{Synthetic dataset}
\label{sec:supp:datasets:synthetic}
In Section 6, a synthetic dataset which was used to explore properties of the suggested WSGC model was presented. Example cases from this dataset are included in Figure~\ref{fig:supp:synthetic}. More specifically, this dataset was utilized to evaluate the contribution of the transitivity closure on the scene-graph-to-layout task. 

Every object in this data is a square with one of two possible sizes, $small$ or $large$. The set of relations includes:
\begin{itemize}
    \item $Above$ - The center of the subject is above the object. This relation is transitive.
    \item $Opposite Horizontally$ - The subject and the object are on opposite sides of the image with respect to the middle vertical line. This relation is not transitive.
    \item $X Near$ - The subject and object are within distance equal to $10\%$ of the image with respect to the $x$ coordinate of each center. This relation is not transitive.
\end{itemize}

To generate SG-layout pairs for training and evaluation, we uniformly sample coordinates of object centers and object sizes and automatically compute relations among object pairs based on their spatial locations.

\begin{figure}[t!]
    \centering
    \includegraphics[width=1\linewidth]{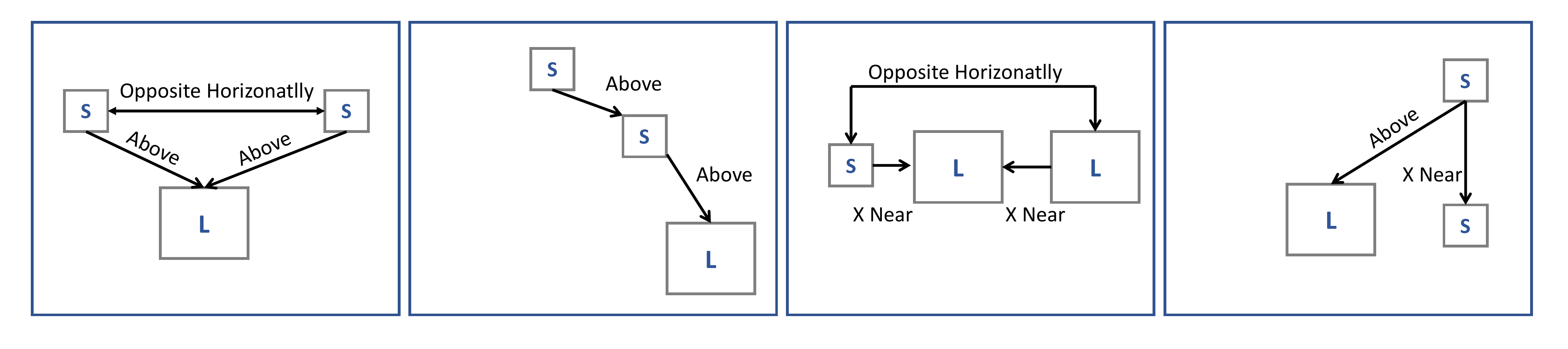}
    \caption{Example of synthetic dataset samples. In these samples, the scene graph relations are overlaid on top of the  ground truth layout. Every edge is described with a corresponding relation type and every square object is annotated with an object type: "S" for small and "L" for large.}
    \label{fig:supp:synthetic}
\end{figure}

\subsection{Packed Datasets} 
\label{sec:supp:datasets:gen}

Here we describe the specific characteristics of the packed datasets presented in the paper. For every packed dataset, only samples with at least $16$ objects per image were included. The method for constructing relations for COCO and CLEVR is as described next. For VG, since Standard VG contains a limited number of relations we supplement the dataset with relations as follows. For every two graph nodes, edges representing geometric relations such as:``left", ``right", ``above", ``below", ``inside" and ``surrounding" are constructed based on relative (x,y) coordinates. Redundant edges are removed such that the graph is minimal. This procedure differs from the one used in \cite{johnson2018image} in two ways: first, in \cite{johnson2018image}, the decision to construct such edges is based on angles between two objects and second, in \cite{johnson2018image}, there can be up to a single constructed edge for every pair of objects and the decision whether to construct or not is random. Hence, the procedure proposed here results in graphs that are more complex w.r.t number of edges and are more informative.

\section{Implementation Details}
\label{sec:supp:impl_details}
\subsection{Scene-Graph-to-layout}
\label{sec:supp:impl_details:sg-to-layout}
In the WSGC GCN model, we follow the implementation details proposed in ~\cite{johnson2018image}. We use 5 hidden layers and an embedding layer of 128 units for each object and relation. The functions $F_s,F_r,F_o$ which were presented in Section 4, are all implemented as a single $3$ layers MLP with $512$ units per layer. For optimization we use Adam \cite{kingma2014adam}, where for $\theta^{conv},\theta^{trans}$ we use LR of $1e^{-2}$ and otherwise we use $1e^{-4}$.

\subsection{AttSPADE}
\label{sec:supp:impl_details:attspade}
We apply Spectral Norm~\cite{miyato2018spectral} to all the layers in both generator and discriminator. We use the ADAM solver~\cite{kingma2015adam} with $\beta_1=0.5$ and $\beta_2=0.999$, and a learning rate of $0.0001$ for both the generator and the discriminator. All the experiments are conducted on NVIDIA V100 GPUs. We use PyTorch synchronized BatchNorm with the following batch sizes: $32$ for $ 128 \times 128$ and $16$ for $256 \times 256$ resolutions (statistics are collected from all the GPUs). The FC layer that calculates $\vv\in\reals^d$ (used to construct tensor $M$. See Section~\ref{sec:supp:attspade}), is set $d$ to $128$.


\section{Proof that SGC outputs the closure $C(E)$ (Section 3.1)}
\label{sec:supp:proof}
\begin{lemma} 
The SGC procedure described in Section 3.1 of the main paper outputs the closure $C(E)$.
\end{lemma}

\begin{proof}
Let $G=(O, E)$. Denote $\hat{C}$ be the canonicalization procedure proposed. To show $\hat{C}(E)=C(E)$, it suffices to prove that (1) $C(E) \subseteq \hat{C}(E)$ and (2) $\hat{C}(E) \subseteq C(E)$.

~\newline\noindent
{Proof that $\hat{C}(E) \subseteq {C}(E)$:}. Let there be $e \in \hat{C}(E)$ s.t $e = (i, r, j)$. We split into cases by $e$ construction:
        \begin{itemize}
            \item \textbf{Original graph edge}. if $e \in E$ then by $C$ definition $e \in C(E)$.
            \item \textbf{Converse constructed edge}. Therefore there exists $r' \in \cR$ such that $(r,r')\in \cR_{conv}$ and $(j, r', i)\in E$. Then $(j, r', i)\in C(E)$ and therefore $(i, r, j)=e\in C(E)$ by definition.
            \item \textbf{Transitive constructed edge}. Since $e$ was constructed in the $Transitivity$ step, it must hold that $r \in \cR_{trans}$ and $e$ was contained in the transitive closure of $r$. Therefore, after the $Converse Relations$ step, there existed a directed path $p=(o_{v_1}, ..., o_{v_k})$ with respect to $r$ where $v_1=i$ and $v_k = j$.  To prove $e \in C(E)$, it is enough to show that for every edge in $p$ it is also in $C(E)$. From here, since $C$ respects transitivity, this will follow. Namely, let there be $e'=(i', r, j') \in \{(o_{v_m}, o_{v_{m+1}}) | m \in \{1, .., k\}\}$. If $e' \in E$, then $e' \in C(E)$ and we are done. Otherwise, by the $Converse Relations$ construction step, there exists $r'$ such that $(r, r')\in \cR_{conv}$ and $(j', r', i')\in E$. Therefore, it follows that $(j', r', i')\in C(E)$ and $e' \in C(E)$ and we are done.
        \end{itemize}
~\newline\noindent
{Proof that $C(E) \subseteq \hat{C}(E)$:} For every $e=(i, r, j) \in C(E)$ we need to show that $e \in \hat{C}(E)$. 
    Since $e \in C(E)$, $e$ is a relation implied by $E$. If $e \in E$, since $\hat{C}$ does not drop edges, it holds that $e\in\hat{C}(E)$ and we're done. Otherwise, we assume by contradiction that $e \notin \hat{C}(E).$ let $p = (o_{v_1}, ..., o_{v_k})$ be a directed path from $o_i$ to $o_j$ in $C(E)$. Then, there exists $e'=(i',r,j') \in  \{(o_{v_{i}}, o_{v_{i+1}})|i \leq k\}$ where $e' \notin \hat{C}(E)$. Otherwise, if there is no such $e'$, we get that there is a directed path between $o_i$ to $o_j$ and by $Transitivity$ step construction $e \in \hat{C}(E)$. Therefore, there must be $e_{conv} \in E$, such that $e_{conv} = (j,r',i)$ and $(r, r') \in \cR_{conv}$. However, from the $Converse Relations$ step construction, if there exists such edge we get that $e \in \hat{C}(E)$, in contrary to the assumption that $e \notin \hat{C}(E)$.
\end{proof}


\section{Generalization on Semantically Equivalent Graphs}
\label{sec:supp:semantic_equiv_graphs}
Results in Table 2 of the main paper demonstrate that the learned WSGC model is more robust to changes in the scene graph input. In this experiment, we randomly transform each test sample scene graph into a semantically equivalent one, and test models on the resulting sample. To generate such samples from a given scene graph, we start by calculating all the possible location-based relations for any pair of objects. Then, for each pair of objects we use prior knowledge to identify pairs of converse relations, and drop one of the edges in such pairs with probability $p=0.5$. After this step, we compute the transitive closure with respect to each relation and randomly drop ($p=0.5$) each edge that does not change the semantics of the scene graph.

\begin{figure*}[t!]
    \centering
    \includegraphics[width=\linewidth]{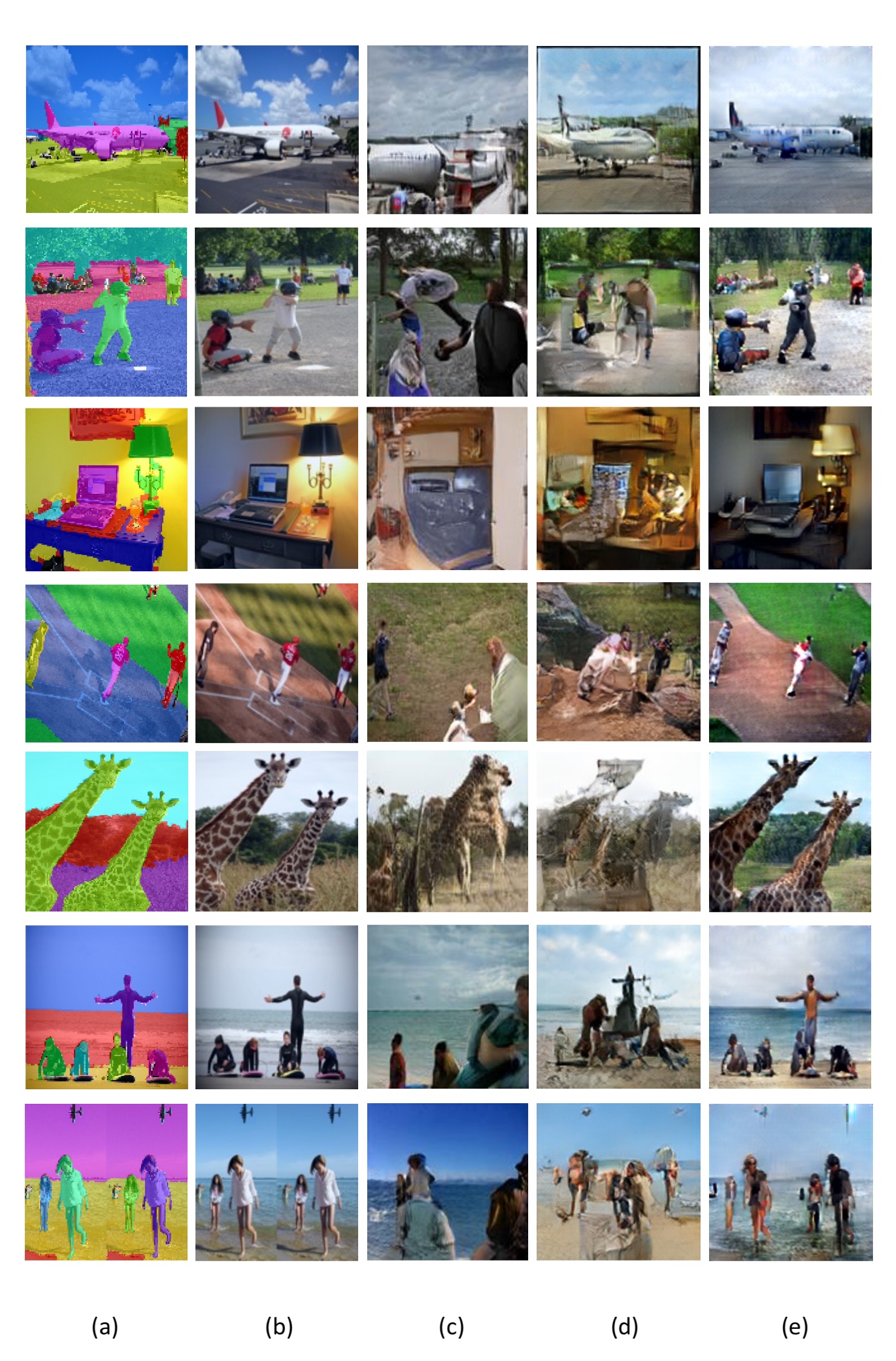}
    \caption{Selected GT layout-to-image generation results on COCO-Stuff dataset on $128\times128$ resultion. Here, we compare our AttSPADE model, Grid2Im~\cite{ashual2019specifying} and LostGAN~\cite{Sun_2019_ICCV} on generation from GT layout of masks. (a) GT layout (only masks). (b) GT image. (c) Generation with LostGAN~\cite{Sun_2019_ICCV} model. (d) Generation with Grid2Im~\cite{ashual2019specifying}. (e) Generation with AttSPADE model (ours).}
    \label{fig:supp:coco128}
\end{figure*}

\begin{figure*}[t!]
    \centering
    \includegraphics[width=\linewidth]{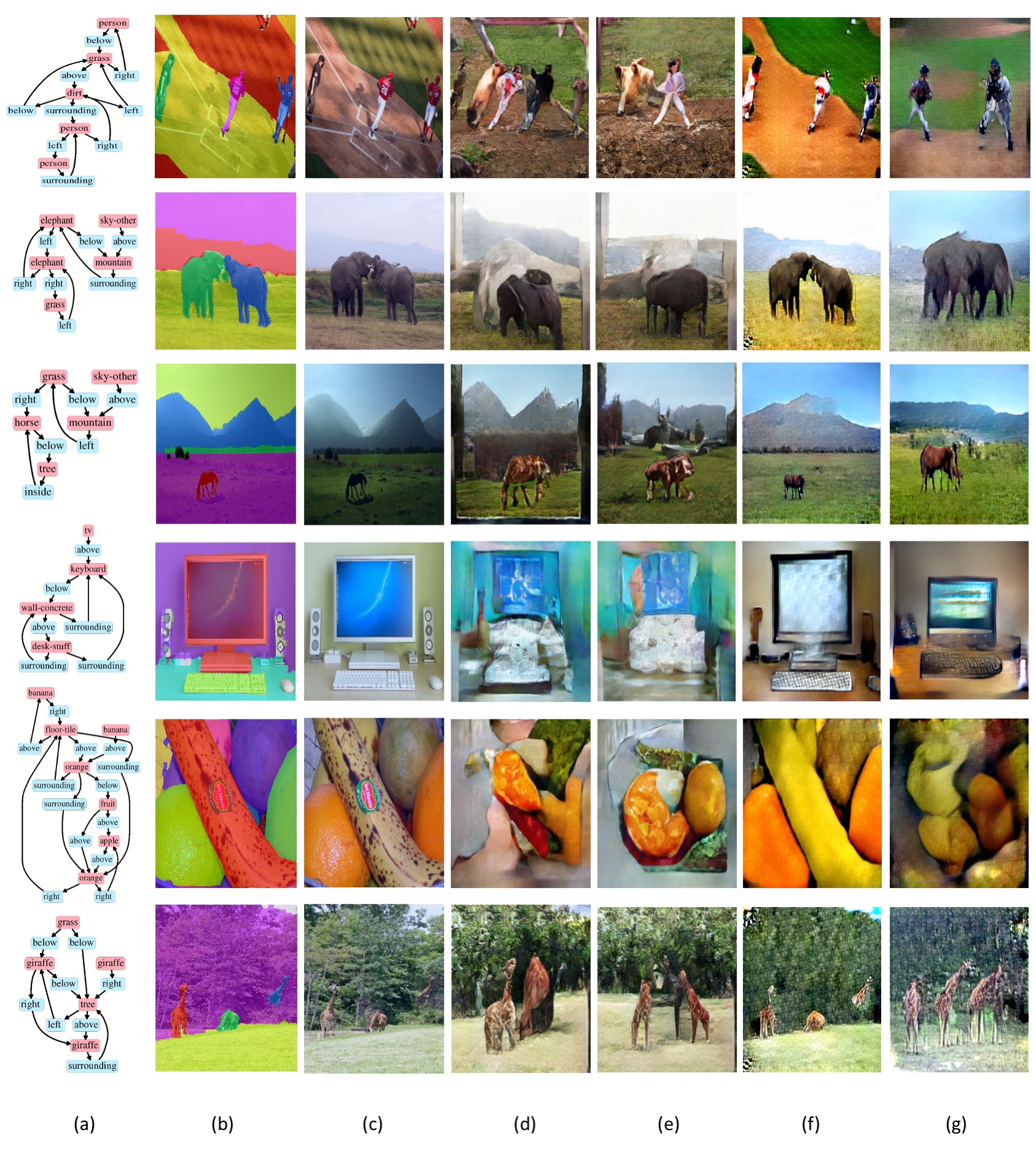}
    \caption{Selected generation results on the COCO-Stuff dataset at $256\times256$ resolution. Here, we compare our AttSPADE model and Grid2Im~\cite{ashual2019specifying} in two different settings: generation from GT layout of masks and generation from scene graphs. (a) GT scene graph. (b) GT layout (only masks). (c) GT image. (d) Generation with Grid2Im~\cite{ashual2019specifying} using the GT layout. (e) Generation with Grid2Im No-att \cite{ashual2019specifying}  from the scene graph (GT layout not used). (f) Generation with AttSPADE model (ours) using the GT layout. (g) Generation with WSGC + AttSPADE model (ours) from the scene graph (GT layout not used).}
    \label{fig:supp:coco256}
\end{figure*}


\begin{figure*}[t!]
    \centering
    \includegraphics[width=\linewidth]{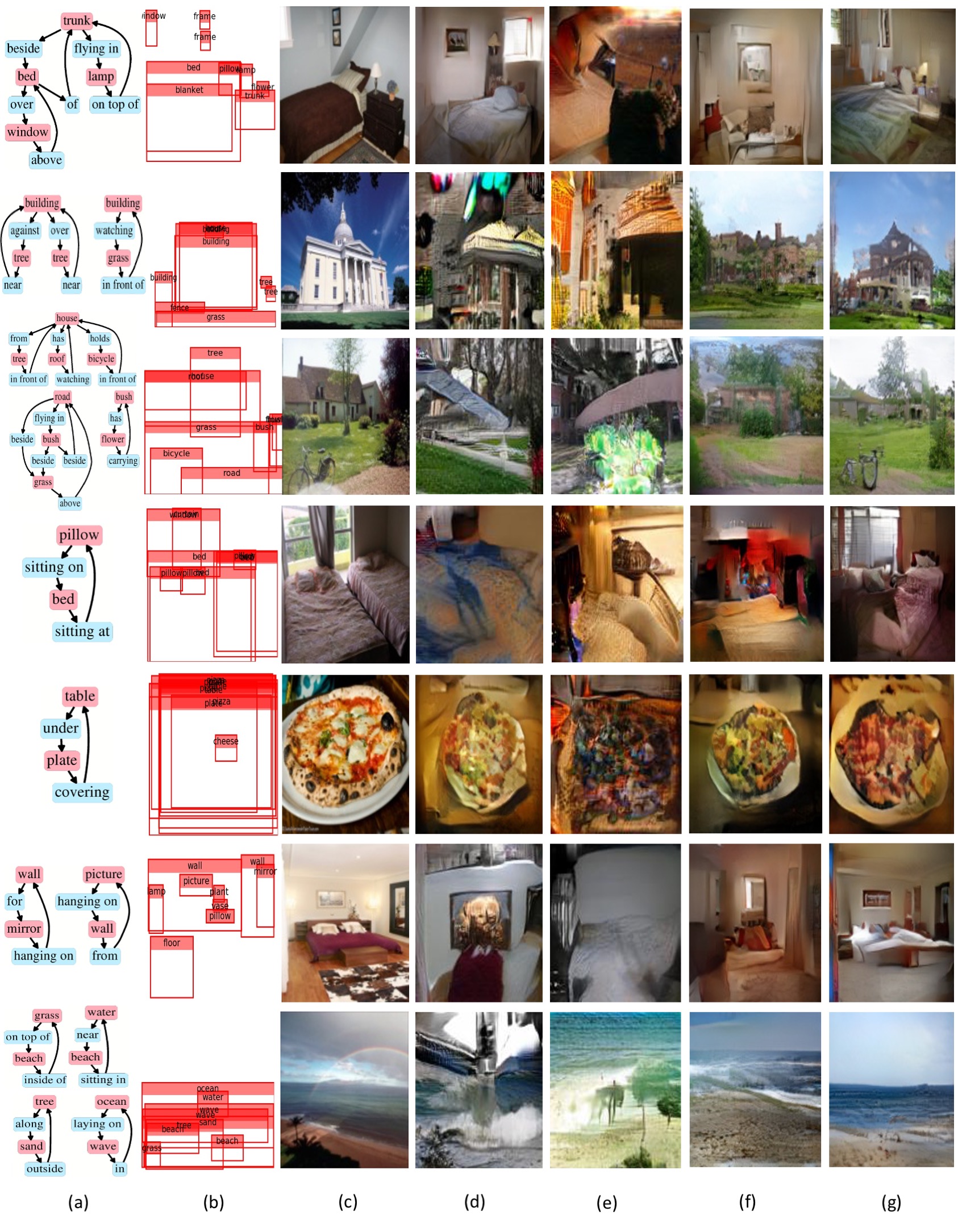}
    \caption{Selected scene-graph-to-image results on Visual Genome dataset on $128\times128$ resolution. Here, we compare our AttSPADE model and LostGAN~\cite{Sun_2019_ICCV} in two different settings: generation from GT layout of boxes and generation from scene graphs. (a) GT scene graph. (b) GT layout (only boxes). (c) GT image. (d) Generation using LostGAN~\cite{Sun_2019_ICCV} from the GT layout. (e) Generation with the WSGC + LostGAN~\cite{Sun_2019_ICCV} from the scene graph (GT layout not used). (f) Generation with the AttSPADE model (ours) from the GT Layout. (g) Generation with the WSGC + AttSPADE model (ours) from the scene graph (GT layout not used).}
    \label{fig:supp:vg128_compare1}
\end{figure*}


\begin{figure*}[t!]
    \centering
    \includegraphics[width=\linewidth]{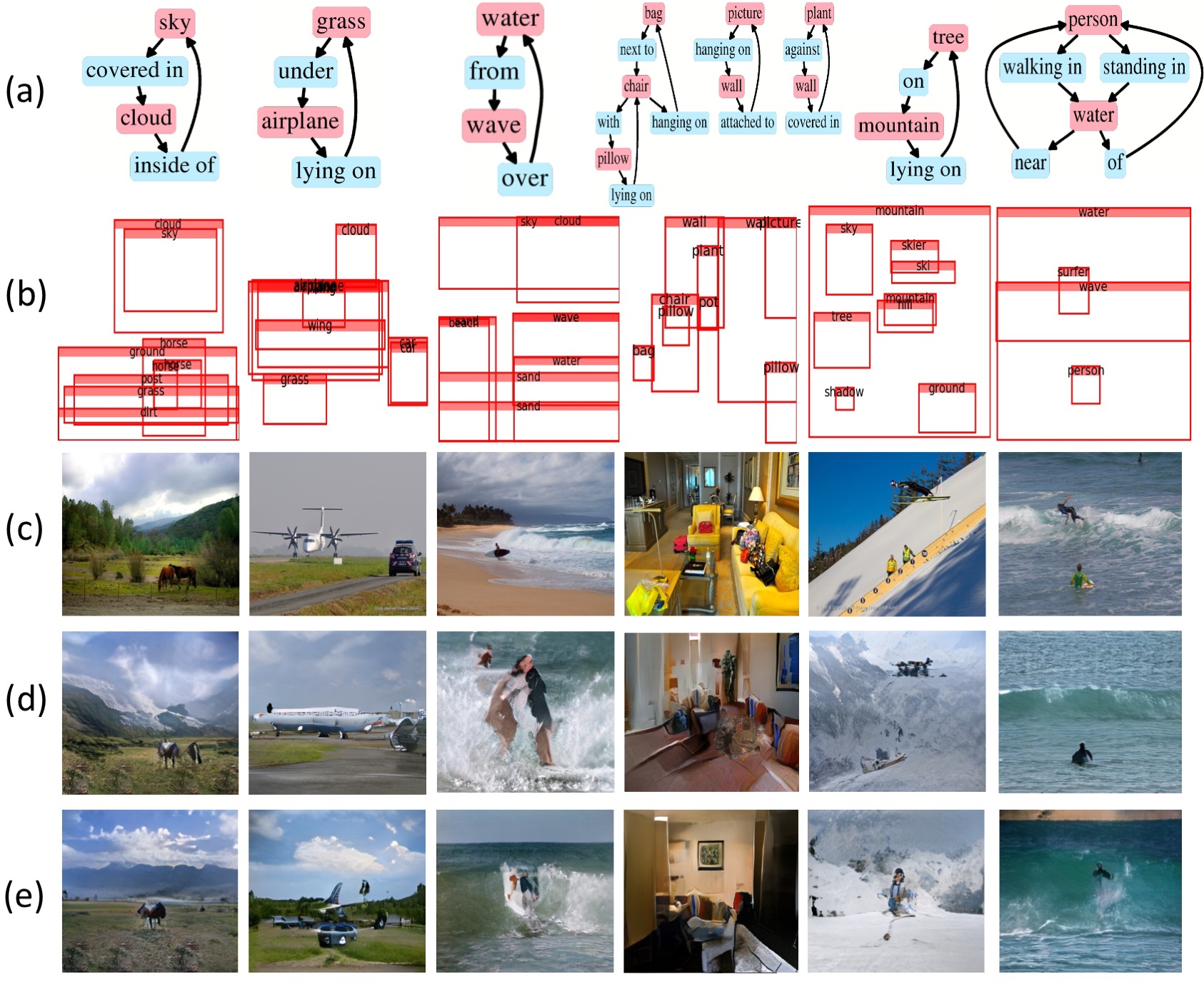}
    \caption{Selected scene-graph-to-image results on the Visual Genome dataset at $256\times256$ resolution. Here, we test our AttSPADE model in two different settings: generation from GT layout of boxes and generation from scene graphs. (a) GT scene graph. (b) GT layout (only boxes). (c) GT image. (d) Generation with the AttSPADE model (ours) from the GT Layout. (e) Generation with the WSGC + AttSPADE model (ours) from the scene graph (GT layout not used).}
    \label{fig:supp:vg256_compare1}
\end{figure*}


\end{document}